\documentclass[11pt]{article}

\usepackage[final]{acl}
\usepackage{times}
\usepackage{latexsym}
\usepackage[T1]{fontenc}
\usepackage[utf8]{inputenc}
\usepackage{microtype}
\usepackage{inconsolata}
\usepackage{amsmath}
\usepackage{amssymb}
\usepackage{booktabs}
\usepackage{placeins}
\usepackage{xspace}

\usepackage{caption}
\usepackage{amsthm}
\usepackage{mathtools}
\usepackage{multirow}
\usepackage{makecell}

\usepackage{titlesec}
\titlespacing*{\section}{0pt}{1.1ex plus 0.2ex minus 0.2ex}{0.7ex plus 0.1ex}
\titlespacing*{\subsection}{0pt}{0.9ex plus 0.2ex minus 0.2ex}{0.5ex plus 0.1ex}

\newtheorem{lemma}{Lemma}
\newtheorem{theorem}{Theorem}

\usepackage{graphicx}
\usepackage{subcaption}
\usepackage{enumitem}
\usepackage{graphicx}
\usepackage{tikz}
\usepackage[english]{babel}
\providecommand*\circled[1]{%
  \tikz[baseline=(char.base)]{
    \ifnum#1>9
    \node[shape=circle,draw,inner sep=0.1pt] (char) {\footnotesize #1};
    \else
    \node[shape=circle,draw,inner sep=0.1pt] (char) {#1};
    \fi
  }
}

\usepackage[most]{tcolorbox}
\tcbset{
  dpbox/.style={
    colback=gray!8,
    colframe=black!20,
    boxrule=0.4pt,
    arc=2pt,
    left=4pt,right=4pt,top=3pt,bottom=3pt
  }
}

\newcommand{\squishlist}{
\begin{list}{$\bullet$}
{   \setlength{\itemsep}{0pt}
   \setlength{\parsep}{3pt}
   \setlength{\topsep}{3pt}
   \setlength{\partopsep}{0pt}
   \setlength{\leftmargin}{1.5em}
   \setlength{\labelwidth}{1em}
   \setlength{\labelsep}{0.5em} } }
\newcounter{Lcount}
\newcommand{\squishlisttwo}{
\begin{list}{\arabic{Lcount}. }
  { \usecounter{Lcount}
 \setlength{\itemsep}{0pt}
 \setlength{\parsep}{0pt}
 \setlength{\topsep}{0pt}
 \setlength{\partopsep}{0pt}
 \setlength{\leftmargin}{2em}
 \setlength{\labelwidth}{1.5em}
 \setlength{\labelsep}{0.5em} } }

\newcommand{\squishend}{\end{list} }

\newcommand{\bicache}{\textsc{BiCache}\xspace}

\DeclareRobustCommand{\ballnumber}[1]{\tikz[baseline=(myanchor.base)] \node[circle,fill=.,inner sep=1pt] (myanchor) {\color{-.}\bfseries\footnotesize #1};}

\title{Enabling KV Caching of Shared Prefix for Diffusion Language Models}

\author{
  Younghun Go\thanks{Equal contribution.}
  \quad
  Jaehoon Han\footnotemark[1]
  \quad
  Changyong Shin
  \quad
  Chuck Yoo\thanks{Corresponding authors.}
  \quad
  Gyeongsik Yang\footnotemark[2]
  \\
  Department of Computer Science and Engineering \\
  Korea University, Seoul, South Korea \\
  \texttt{yhgo@os.korea.ac.kr, hjhoon03@korea.ac.kr, cyshin@os.korea.ac.kr} \\
  \texttt{chuckyoo@os.korea.ac.kr, g\_yang@korea.ac.kr}
}

\begin{document}
\maketitle
\enlargethispage{2\baselineskip}
\begin{abstract}
Key-value (KV) caching for shared prefixes is essential for high-throughput large language model (LLM) serving, but it faces critical challenges in emerging diffusion language models (DLMs). In DLMs, bidirectional attention means that updating any token dynamically alters the entire context and its corresponding KVs. Thus, existing caching techniques developed for LLMs, which assume that KVs remain invariant once computed, corrupt the shared prefix KVs. Our experiments show that applying these techniques to DLMs causes model accuracy to collapse to near zero.

To unlock high-throughput DLM serving, we propose bidirectional prefix caching, \bicache, the first KV caching technique for shared prefixes in DLMs. \bicache is designed based on key observations from our comprehensive analysis: shared prefix KVs remain stable and reusable in shallow layers, while the depth of shallow layers depends on the fraction of shared prefix tokens in each request. Thus, \bicache dynamically identifies a safe layer depth for reusing shared prefix KVs and eliminates redundant computation. Evaluations demonstrate that \bicache significantly improves serving throughput by 36.3\%--98.3\% compared to existing techniques without accuracy collapse (only 0--1.8\% difference). The code and dataset are publicly available
at \url{https://github.com/OSSS-KU/BiCache}.

\end{abstract}

\section{Introduction}
Key-value (KV) caching is a fundamental optimization for efficient large language model (LLM) serving. In attention layers, token hidden states are projected into key and value tensors, and storing these tensors allows serving systems to reuse them in later generation steps instead of recomputing them, thereby improving serving latency and throughput. KV caching is particularly important for shared prefixes, where the same token sequence (e.g., system prompts) appears at the beginning of multiple requests and its KVs can be cached and reused across those requests. In real-world workloads, shared prefixes account for 85--97\% of prompt tokens \cite{preble}.

Most existing LLMs are based on autoregressive models (ARMs) that generate tokens sequentially \cite{chatgpt}. In ARMs, the KVs of previously generated tokens remain unchanged once computed. As a result, modern serving systems use ``shared prefix caching,'' which stores the KVs of a shared prefix once and reuses them across requests \cite{vllm}, making shared prefix caching a key optimization for high-throughput serving.

Recently, diffusion language models (DLMs) have emerged as a promising alternative to ARMs \cite{llada}. ARMs generate tokens sequentially, making latency grow with output length \cite{longform}. Moreover, their causal attention can lead to the reversal curse, where a model may correctly answer a relation in left-to-right but fail to infer its inverse \cite{curse, rethinking-curse}. In contrast, DLMs generate tokens at arbitrary positions through iterative denoising steps with bidirectional attention, allowing all positions to interact at each step and alleviating the strict left-to-right generation bottleneck of ARMs.

However, the bidirectional generation process of DLMs makes shared prefix caching much harder. Under bidirectional attention, updating one token changes the KVs of all other tokens, including those in the shared prefix. Therefore, shared prefix KVs are not invariant across steps or across requests, which is a key difference from ARMs. Directly reusing the shared prefix KVs as in ARMs can corrupt the attention context and severely collapse model accuracy. Our motivating experiments show that naively applying shared prefix caching designed for ARMs to DLMs drives accuracy collapse to near zero on diverse benchmarks (\S\ref{sec:3}).

Recent studies \cite{fastdllm, fastdllmv2, dllmcache} have introduced KV caching techniques for DLMs; however, they focus on reusing KVs within a single request and do not address shared prefix caching across requests. Because shared prefixes dominate practical workloads and are central to serving throughput, enabling shared prefix caching for DLMs remains a major open challenge.

To address this challenge, we propose bidirectional prefix caching, \bicache, the first shared prefix caching technique for DLMs to our knowledge. The key idea of \bicache is to selectively reuse KVs according to layer depth and the shared prefix ratio of a request. Through an in-depth analysis of how shared prefix KVs change across requests, layers, and denoising steps, we make three important observations (\S\ref{sec:4}). First, shared prefix KVs remain highly similar across requests in shallow layers, making shared prefix caching possible in those layers. Second, the depth of shallow layers that can safely reuse cached KVs is strongly correlated with the shared prefix ratio, i.e., the fraction of shared prefix tokens in each request. Third, although the remaining deep layers cannot directly reuse shared prefix KVs across requests, they can still reduce redundant computation within a request by reusing KVs with periodic refresh (\S\ref{sec:5}).

Based on the observations, \bicache introduces two practical mechanisms: shared prefix profiling and layer-partitioned caching. Shared prefix profiling determines the layer depth up to which shared prefix KVs can be safely reused without harming accuracy. During serving, layer-partitioned caching directly reuses cached KVs of the shared prefix in shallow layers and applies periodic refresh in deep layers to remove redundant KV computation.

The major contributions of this study are:
\squishlist
\item Characterize how shared prefix KVs change in DLMs across requests, layers, and steps.
\item Present \bicache, to our knowledge the first shared prefix caching technique for DLMs.
\item Demonstrate significant throughput improvements by 36.3\%--98.3\% without accuracy collapse (only 0--1.8\% differences)

\squishend

\section{Background}

\subsection{DLM}

LLaDA \cite{llada} is a transformer-based DLM and a foundation for many later DLMs. Fig. \ref{fig:workflow} shows its generation process. It has three key hyperparameters (\ballnumber{0}): the number of transformer layers $L$, the output length $g$, and the number of denoising steps $s$. Throughout the paper, we use the term ``step'' to refer to a denoising step. Over $s$ steps, the model generates $g$ tokens, predicting either $\lfloor g/s \rfloor$ or $\lfloor g/s \rfloor{+}1$ tokens per step (e.g., $g{=}8$, $s{=}3$ gives $\{3,3,2\}$ tokens per step).

When a new request arrives, LLaDA forms an input string by appending $g$ \texttt{[MASK]} tokens to the prompt tokens (\ballnumber{1}). At each step, the input string is processed by a forward pass through the $L$ transformer layers (\ballnumber{2}). In each step, DLM applies bidirectional attention and denoises the designated number of \texttt{[MASK]} tokens so that the predicted words (tokens) are filled in. In Fig. \ref{fig:workflow}, step 1 denoises three tokens (as explained above), so the first step denoises ``explains'', ``prefix'', and ``caching'' (\ballnumber{3}). The input string with three denoised tokens is passed to the next step, where another denoising happens using the $L$ transformer layers. The denoised tokens from the previous step remain unchanged, and the model denoises the designated number of mask tokens, three in the example, producing ``clearly'', ``why'', and ``impossible.'' This iterative process continues (\ballnumber{4}) for $s$ steps until all \texttt{[MASK]} tokens are denoised.

\begin{figure}[t]
\centering
\includegraphics[width=\linewidth]{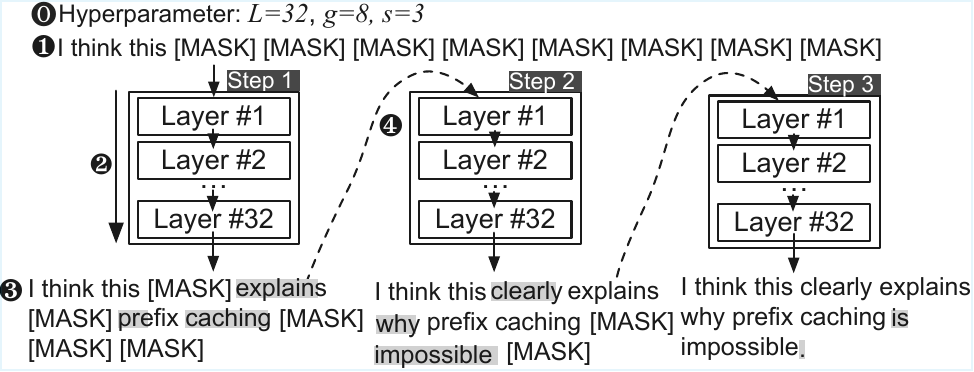}
\caption{Workflow of token generation in DLMs.}
\label{fig:workflow}
\end{figure}

\subsection{Key-Value and Bidirectional Attention}\label{sec:2.2}

Here, we detail the DLM layer. We explain how hidden states are projected into KVs and how bidirectional attention aggregates them to update token representations. Finally, we explain how this bidirectional nature causes KVs to change across steps.

\noindent\textbf{KV calculation.}
At transformer layer $\ell \in \{1,{\ldots},L\}$ and step $t \in \{1,{\ldots},s\}$, the hidden state of a token is denoted by $\mathbf{h}_{\ell,t} \in \mathbb{R}^d$, where $d$ is the hidden dimension (the feature size of each token representation).
The layer computes the key and value vectors using its projection matrices $W^K_{\ell} \in \mathbb{R}^{d_k \times d}$ and $W^V_{\ell} \in \mathbb{R}^{d_v \times d}$. Here, $d_k$ and $d_v$ denote the dimensions of the key and value vectors, respectively; the key vector ($\mathbf{k}_{\ell,t} \in \mathbb{R}^{d_k}$) is used to compute attention weights, and the value vector ($\mathbf{v}_{\ell,t} \in \mathbb{R}^{d_v}$) is aggregated according to those weights. The computation is
\begin{equation}\label{eq:1}
\setlength{\abovedisplayskip}{2pt}
\setlength{\belowdisplayskip}{2pt}
\mathbf{k}_{\ell,t} = W^K_{\ell}\mathbf{h}_{\ell,t}, \quad
\mathbf{v}_{\ell,t} = W^V_{\ell}\mathbf{h}_{\ell,t}.
\end{equation}
For conciseness, throughout the paper, we denote the key--value pair by $\mathbf{KV}_{\ell,t} \triangleq (\mathbf{k}_{\ell,t}, \mathbf{v}_{\ell,t})$.

\noindent\textbf{Bidirectional attention.}
Using the keys and values, layer $\ell$ computes an attention output $\mathbf{o}_{\ell,t}$ to gather context from the input string. Specifically, it is calculated using a query vector $\mathbf{q}_{\ell,t}$ (also projected from $\mathbf{h}_{\ell,t}$ using $W^Q_{\ell}$) that is compared against the keys of all tokens in the input string, where $n$ denotes the number of tokens in the input string:
\begin{equation}\label{eq:2}
\setlength{\abovedisplayskip}{4pt}
\setlength{\belowdisplayskip}{-2pt}
\mathbf{o}_{\ell,t} =
\sum_{j=1}^{n}
\mathrm{softmax}_{j}\!\left(
\frac{\mathbf{q}_{\ell,t} \cdot \mathbf{k}^{(j)}_{\ell,t}}{\sqrt{d_k}}
\right)
\mathbf{v}^{(j)}_{\ell,t}.
\end{equation}

Here, $\mathbf{o}_{\ell,t}$ is a weighted sum of value vectors $\mathbf{v}^{(j)}_{\ell,t}$, representing aggregated global context. The weights are given by the softmax-normalized scaled dot products, measuring the relevance of the $j$-th token to the current token at step $t$. To produce each token's $\mathbf{o}_{\ell,t}$, the query attends to all $n$ tokens.
This output $\mathbf{o}_{\ell,t}$ updates the next layer's hidden state $\mathbf{h}_{{\ell{+}1},t}$. After $L$ layers, the final hidden state is transformed into vocabulary logits to predict the most likely word (e.g., ``explains'' in Fig.~\ref{fig:workflow}).

Importantly, with bidirectional attention, denoising a token during a step changes $\mathbf{o}_{\ell,t}$ for other tokens. This propagates to their hidden states and modifies their key and value vectors via Eq.~(\ref{eq:1}). For example, in Fig.~\ref{fig:workflow}, even after the token ``explains'' has been denoised, its KVs continue to be updated in later steps as the surrounding tokens change.

\subsection{Shared Prefix and Its Caching}\label{sec:2.3}

Shared prefixes are common in practice. For example, many production LLM serving systems prepend a fixed system prompt to every user request to provide formatting guidelines, safety and ethical constraints, and role instructions \cite{systemprompt}. See Appendix \ref{app:1} for examples of real-world shared prefixes.
Also, in multi-turn conversations, shared prefixes naturally arise. When a user submits a new request within the same session, the input string to the model typically includes previous user requests and responses as a shared prefix so that the model can maintain conversational context \cite{multiturn}.
As a result, a substantial portion of the prompt tokens is repeated across requests, accounting for up to 97\% of the total prompt token length in practical deployments \cite{preble}.

Shared prefix caching avoids redundant KV computation by 1) computing the KVs of a shared prefix once and 2) reusing them for later requests whose inputs contain the same prefix. In ARMs \cite{vllm}, once the KVs of the shared prefix are cached, they can be reused directly for other requests at every layer and across all steps. This is possible because, under causal attention in ARMs, the KV values of the prefix remain unchanged even when additional suffix tokens are appended. Prior work reports that shared prefix caching in ARMs can provide $\sim$10$\times$ latency improvement, making it an essential optimization for practical LLM serving systems \cite{promptcache}.

\section{Motivation: New Shared Prefix Caching Technique is Required for DLMs}\label{sec:3}
In this section, we apply existing shared prefix caching designed for ARMs to DLMs and show that it causes high accuracy drop, which motivates the need for a new technique.

\subsection{Setup}\label{sec:3.1}
All experiments are conducted on an NVIDIA B200 GPU with 180 GB memory. We evaluate LLaDA \cite{llada} that has $L{=}32$. We set $g{=}256$, and $s{=}128$; thus, two tokens are unmasked at each step. We use a batch size of 1 to fit the model in GPU memory, which is consistent with the configuration used in other studies \cite{d3llm,tidar}.

LLaDA is evaluated on widely used benchmarks: 1) ARC-Challenge-Chat (science multiple-choice QA), 2) GPQA (graduate-level question answering), and 3) MATH-500 (competition-style mathematics). Similar to existing LLM serving systems and studies \cite{relayattention, gsm8k}, we prepend two system prompts to each request. First, we attach the real-world system prompt of xAI’s Grok-4 model \cite{grok4} (details in Appendix \ref{app:1}), which provides role and safety instructions. Second, we prepend another system prompt that enforces a specific output format (i.e., attaching \texttt{\#\#} before the final answer) required for strict answer matching in the benchmarks (described below). Across benchmarks, the shared prefix accounts for 85\% of the request on average (details in Appendix \S\ref{app:ratio}), consistent with real-world workloads~\cite{preble}.

We evaluate two metrics: 1) throughput and 2) accuracy. Throughput is measured as the number of generated tokens per second (tokens/s), averaged over all requests in each benchmark. For accuracy, following prior work \cite{strictmatch}, we apply strict matching between the model output and the ground-truth answer. From the LLaDA output, the final answer following the \texttt{\#\#} symbol is considered correct only if it exactly matches the ground-truth answer in each benchmark. Accuracy is reported as the percentage of correct outputs over all requests.

We compare the following baselines:
\squishlist
  \item \textbf{No caching}: LLaDA inference without any shared prefix caching.
  \item \textbf{vLLM~\cite{vllm}}: LLaDA with vLLM, a representative shared prefix caching technique for ARMs. For requests sharing the same prefix, vLLM computes the KVs from the shared prefix once, and then reuses them across requests for all layers and steps.
\squishend

{\captionsetup[figure]{skip=3pt}
\begin{figure}[t]
\centering
\begin{subfigure}[t]{0.45\linewidth}
\centering
\includegraphics[width=\linewidth]{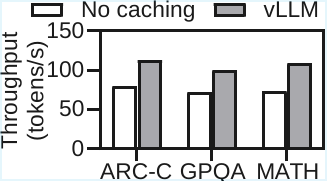}
\caption{Throughput.}
\label{fig:2a}
\end{subfigure}
\hfill
\begin{subfigure}[t]{0.44\linewidth}
\centering
\includegraphics[width=\linewidth]{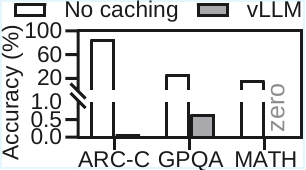}
\caption{Accuracy.}
\label{fig:2b}
\end{subfigure}
\caption{Motivating experiment.}
\label{fig:2}
\end{figure}
}

\subsection{Motivating Experiment Results}\label{sec:3.2}
Fig. \ref{fig:2a} shows the throughput across benchmarks.
Compared to no caching, vLLM consistently achieves higher throughput on all benchmarks. Specifically, vLLM improves throughput by $\sim$42.7\%, $\sim$38.5\%, and $\sim$48.5\% on ARC-Challenge-Chat, GPQA, and MATH-500.

However, existing techniques exhibit severe accuracy collapse. Fig. \ref{fig:2b} shows the accuracy on the same benchmarks. In contrast to the throughput gains in Fig. \ref{fig:2a}, vLLM causes accuracy collapse across all benchmarks. Specifically, while no caching achieves 86.1\%, 27.9\%, and 18.2\% accuracy on the three benchmarks, vLLM shows near-zero accuracy (0.1\%, 0.7\%, and 0\%), indicating that it fails to generate correct answers. The results demonstrate that, although the shared prefix caching technique has potential to increase throughput, the technique designed for ARMs leads to significant accuracy collapse because the KV values of shared prefix in DLMs continuously change.

{\setlength{\abovecaptionskip}{2pt}
\begin{figure*}[!t]
\centering
\begin{minipage}[c]{0.27\linewidth}
  \centering
  \includegraphics[width=\linewidth]{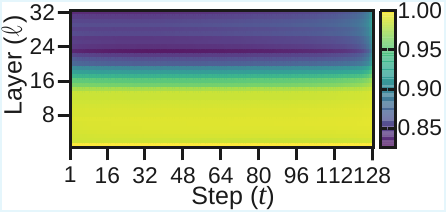}
  \captionof{figure}{Heatmap of $\mathbf{sim}_{\ell,t}$ for all layers and steps.}
  \label{fig:3}
\end{minipage}\hfill%
\begin{minipage}[c]{0.45\linewidth}
  \refstepcounter{figure}
  \centering
  \begin{subfigure}[t]{0.49\linewidth}
    \centering
        \includegraphics[width=\linewidth]{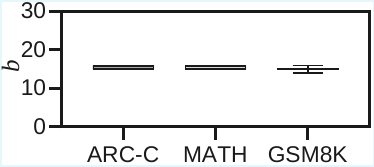}
    \subcaption{Per task type.}
    \label{fig:4b}
  \end{subfigure}\hfill%
  \begin{subfigure}[t]{0.49\linewidth}
    \centering
    \includegraphics[width=\linewidth]{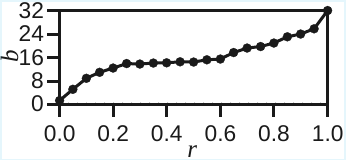}
    \subcaption{Per shared prefix ratio.}
    \label{fig:4a}
  \end{subfigure}
  \caption*{Figure \thefigure: Shallow layer depth determination.}
\end{minipage}
\begin{minipage}[c]{0.25\linewidth}
  \centering
  \includegraphics[width=\linewidth]{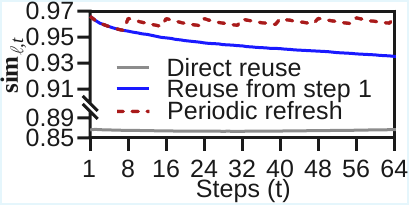}
  \captionof{figure}{Deep layers $\mathbf{sim}_{\ell,t}$ under periodic refresh.}
  \label{fig:5}
\end{minipage}
\end{figure*}
}
\section{Key Observations}\label{sec:4}
To design shared prefix caching for DLMs, we present three key observations (O1--O3) on KV value changes in the shared prefix.

\subsection{KV similarity is Layer-dependent but Stable across steps}\label{sec:4.1}
In ARMs, shared prefix caching is feasible because the KVs of the shared prefix remain identical across requests, allowing the KVs computed once to be reused at all layers and steps. In DLMs, however, bidirectional attention changes KV values of shared prefix during denoising, so it is unclear whether shared prefix KV reuse across requests is still possible. To answer this, we examine i) whether the shared prefix KVs from one request remain similar in other requests, and ii) how this similarity varies across layers and steps. We follow the same setup as in \S\ref{sec:3.1}, and use WildChat-4.8M~\cite{wildchat}, a large-scale conversational dataset with diverse prompt lengths and tasks. Because the dataset lacks ground-truth answers, we do not use it for evaluation as in Fig.~\ref{fig:2}. Instead, we use it to analyze the KV computation behavior of DLM.

We measure the cosine similarity between 1) the shared prefix KVs over all shared prefix token positions computed from the shared prefix alone, which are cached for reuse across requests, and 2) the corresponding shared prefix KVs at the same token positions obtained when each request is processed without caching during denoising, across layers and steps. For layer $\ell \in \{1,2,{\ldots},L\}$ and step $t \in \{1,2,{\ldots},s\}$, we denote them by $\mathbf{KV}^{\$}_{\ell}$ and $\mathbf{KV}^{\mathrm{ref}}_{\ell,t}$, respectively.
\footnote{The cached KVs do not carry a step index $t$ as they are computed once at the first step, as in vLLM, and then reused across all denoising steps. Each KV representation is flattened into a single vector before computing cosine similarity.} We define
\begin{equation}\label{eq:kv_sim_step}
\setlength{\abovedisplayskip}{2pt}
\setlength{\belowdisplayskip}{2pt}
\mathbf{sim}_{\ell,t} \triangleq \cos\!\Big(\mathbf{KV}^{\$}_{\ell},\, \mathbf{KV}^{\mathrm{ref}}_{\ell,t}\Big).
\end{equation}
The value of $\mathbf{sim}_{\ell,t}$ ranges from $-1$ to $1$, where values closer to $1$ indicate higher similarity.

Fig. \ref{fig:3} shows a heatmap of $\mathbf{sim}_{\ell,t}$ over layers and steps, averaged across requests. Although bidirectional attention makes the similarity lower than 1, the results reveal two patterns that make shared prefix caching feasible in DLMs.
First, $\mathbf{sim}_{\ell,t}$ is much higher in shallow layers than in deep layers: shallow layers, $\ell \in \{1,{\ldots},16\}$, show an average similarity of 0.98, whereas deep layers, $\ell \in \{17,{\ldots},32\}$, show 0.86. This indicates that shared prefix KV reuse across requests is feasible mainly in shallow layers.

Second, within each layer, $\mathbf{sim}_{\ell,t}$ remains nearly constant across steps. The standard deviation across steps is only 0.007 on average and at most 0.016, indicating little change in similarity over steps.

\begin{tcolorbox}[dpbox]
\textbf{O1:} In shallow layers, shared prefix KVs are highly similar and remain stable across steps. Therefore, shared prefix KVs in shallow layers can be cached and reused across requests.
\end{tcolorbox}

\subsection{Shallow layer depth is correlated with shared prefix ratio}\label{sec:4.2}
O1 shows that KV reuse across requests is feasible in shallow layers. We next investigate how many layers can be treated as shallow for a given input. To quantify this depth, we define $b$ as the largest layer depth such that $\mathbf{sim}_{\ell,t} \ge \tau$ holds for all layers $\ell \in \{1,{\ldots},b\}$ and all steps $t \in \{1,{\ldots},s\}$, where $b \le L$. In other words, layers $\{1,{\ldots},b\}$ are the shallow layers whose shared prefix KVs can be reused without accuracy collapse.\footnote{The first layer's KVs are projected from the hidden states before bidirectional attention (\S\ref{sec:2.2}); thus the shared prefix KVs at the first layer are identical to those computed during denoising, so always their similarity is 1 and hence $b \ge 1$.} Here, $\tau$ is the similarity threshold above which $\mathbf{KV}^{\$}_{\ell}$ can be reused without accuracy loss. From our analysis, we set $\tau {=} 0.97$ through empirical analysis (to be explained in \S\ref{sec:6.6}).

For each DLM, two input characteristics can affect $b$: 1) task type, such as math, coding, or general QA, and 2) shared prefix ratio $r$, i.e., the fraction of shared prefix tokens in the input string. We examine how these two characteristics affect $b$.

First, we analyze how $b$ varies across task types. We measure $b$ as Eq. (\ref{eq:kv_sim_step}) on ARC-Challenge-Chat, MATH-500, and GSM8K when $r$ is 0.9. Different $r$ values also show similar results. Fig. \ref{fig:4b} shows that $b$ remains consistent across three benchmarks, with a standard deviation of only 0.43. This indicates that task type is not the main factor determining the shallow layer depth.

Second, we analyze how $b$ changes with $r$. Fig.~\ref{fig:4a} shows that $b$ increases as $r$ increases. The reason is that, under bidirectional attention, the KV of each shared prefix token is affected by the entire input string, including both shared prefix tokens and non-shared tokens. When $r$ is small, non-shared tokens occupy a larger fraction of the input, so their effect on the shared prefix KVs is larger, and the similarity drops below $\tau$ at earlier layers. In contrast, when $r$ is large, the input is dominated by shared prefix tokens, so the influence of non-shared tokens decreases. As a result, the shared prefix KVs remain similar across requests for more layers, which increases $b$. We theoretically prove this relationship in Appendix \S\ref{app:2}.\footnote{$b$ also remains stable and consistent under small changes in system prompt tokens (e.g., instruction wording or metadata fields); see Appendix \S\ref{app:sys}.}

{
\begin{tcolorbox}[dpbox]
\textbf{O2:} The shallow layer depth $b$ is correlated with the shared prefix ratio $r$.
\end{tcolorbox}
}

{
\begin{figure*}[t]
\centering
\includegraphics[width=.9\linewidth]{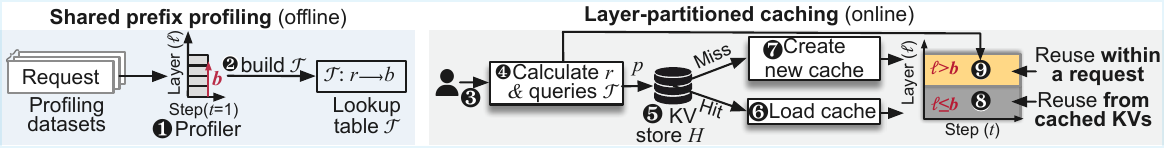}
\caption{Overview of \bicache. \bicache performs one-time offline profiling to build a lookup table that maps each shared prefix ratio to the maximum depth up to which KVs can be safely reused. At runtime, \bicache uses the shared prefix ratio of each incoming request to determine the reusable depth from the lookup table. \bicache then reuses cached KVs for the shared prefix across requests in the selected shallow layers. In deeper layers, \bicache recomputes and periodically refreshes KVs within each request.}
\label{fig:dspc_overview} 
\end{figure*}
}

\subsection{Deep layers can be cached \emph{within} a request}\label{sec:4.3}

O1 and O2 show that reuse of shared prefix KVs across requests is feasible for shallow layers $\ell \in \{1,{\ldots},b\}$. We now examine whether the remaining deep layers $\ell \in \{b{+}1,{\ldots},L\}$ can still benefit from KV caching. To do so, we use the same similarity metric $\mathbf{sim}_{\ell,t} \triangleq \cos\!\big(\mathbf{KV}^{\$}_{\ell}, \mathbf{KV}^{\mathrm{ref}}_{\ell,t}\big)$, but now analyze it for deep layers within a request.

Fig. \ref{fig:5} compares three settings over the deep layers, while reusing cached KVs for the shallow layers as in O1. First, when $\mathbf{KV}^{\$}_{\ell}$ is directly reused from another request, as in Fig. \ref{fig:3}, $\mathbf{sim}_{\ell,t}$ remains low, with an average of 0.86 (black line). Second, when $\mathbf{KV}^{\$}_{\ell}$ is computed once from the current request at the first step of each layer and then reused for later steps, $\mathbf{sim}_{\ell,t}$ gradually decreases as $t$ increases, reaching 0.93 on average (blue line), because the cached $\mathbf{KV}^{\$}_{\ell}$ becomes stale.

Third, based on this observation, we test a periodic-refresh design in which $\mathbf{KV}^{\$}_{\ell}$ is recomputed for each layer of the current request every 8 steps and reused only until the next refresh. This design keeps $\mathbf{sim}_{\ell,t}$ high at 0.96 (red dotted line), close to the threshold of 0.97 for shallow layers, showing that deep layers can still support effective KV reuse within a request when stale KVs are periodically refreshed.

\begin{tcolorbox}[dpbox]
\textbf{O3:} For deep layers, shared prefix KVs can be reused \textit{within} a request for multiple steps when they are periodically refreshed.
\end{tcolorbox}

\section{Design}\label{sec:5}\vspace{-0.25em}
Fig.~\ref{fig:dspc_overview} presents the workflow of \bicache, which consists of two mechanisms: 1) shared prefix profiling and 2) layer-partitioned caching. Given a DLM to be served, \bicache first performs offline shared prefix profiling to determine the shallow layer depth $b$ for different shared prefix ratios $r$. Based on O1, it measures $\mathbf{sim}_{\ell,t}$ at the first step for each request and layer using predefined profiling datasets (\ballnumber{1}). Based on O2, it then builds a lookup table $\mathcal{T}$ that maps each $r$ to the corresponding shallow layer depth $b$ (\ballnumber{2}). This profiling is performed offline once per model before serving user requests.

When a user request arrives online (\ballnumber{3}), \bicache performs layer-partitioned caching. It first computes $r$ for the request and queries $\mathcal{T}$ to obtain $b$ (\ballnumber{4}). \bicache then retrieves the cached shared prefix KVs. The cached KVs are stored in KV store $H$ (\ballnumber{5}) that maps each shared prefix string $p$ to its cached KVs. For a new request, \bicache extracts the shared prefix $p$ and looks up the corresponding entry in $H$.

On a cache hit, \bicache loads the cached KVs and reuses them for shallow layers $1$ to $b$ (\ballnumber{6}). On a miss, it creates a new entry in $H$ (\ballnumber{7}) by computing the KVs of $p$ on the DLM for all layers.

For each layer $\ell$, \bicache then determines whether $\ell \le b$ (shallow layers) or $\ell > b$ (deep layers). If $\ell \le b$, \bicache directly reuses the cached shared prefix KVs (\ballnumber{8}). Otherwise, following O3, it computes the shared prefix KVs within the current request and reuses them across steps with periodic refresh every $\Delta$ steps (\ballnumber{9}). We explain each mechanism in detail below.

\subsection{Shared Prefix Profiling}\label{sec:5.1}
We explain 1) profiling dataset, 2) profiler, and 3) lookup table $\mathcal{T}$.

\noindent\textbf{Profiling dataset.}
We prepare a profiling dataset to learn the relationship between $b$ and $r$. We use WildChat-4.8M \cite{wildchat} as a source of diverse user prompts and construct requests with different $r$ by combining each prompt with the shared prefix (system prompt \cite{relayattention}). For each request, $r$ is calculated as the ratio of shared prefix tokens to all tokens in the input string.
From the requests, we uniformly sample $M$ requests for each $r \in \{0.01, {\ldots}, 1.00\}$, resulting in $100 \times M$  requests. We set $M=500$ empirically (\S\ref{sec:6.6}).

\noindent\textbf{Profiler.}
Using the profiling dataset, \bicache profiles each request as follows. We follow the same method to measure KV similarity used in Fig. \ref{fig:3}. For each request with ratio $r$, we prepare two inputs: 1) the string with shared prefix only and 2) the full input of the shared prefix and remaining prompt tokens. Running both inputs at the first denoising step for every layer yields two KV values for the shared prefix: $\mathbf{KV}^{\$}_{\ell}$ (cached KV) and $\mathbf{KV}^{\mathrm{ref}}_{\ell,1}$ (KV computed per layer without caching). We then calculate $\mathbf{sim}_{\ell,1}$ using Eq. \eqref{eq:kv_sim_step}.

\noindent\textbf{Lookup table $\mathcal{T}$.}
From the profiled $\mathbf{sim}_{\ell,1}$ values, \bicache defines $b$ as the largest layer index such that all layers from $1$ to $b$ satisfy $\mathbf{sim}_{\ell,1} \ge \tau$, where $\tau = 0.97$ (determined empirically in \S\ref{sec:6.6}).

To determine $b$ for each $r$, we use the profiled results on $M$ requests for every $r \in \{0.01, {\ldots}, 1.00\}$. Each request $i$ produces a set of $\mathbf{sim}_{\ell,1}$ values across layers, from which \bicache determines $b_i$ as the largest layer index satisfying the threshold $\tau$. \bicache then computes the average of $\{b_i\}_{i=1}^{M}$ and uses it as  representative $b$ for ratio $r$. Repeating this procedure for all $r$ constructs $\mathcal{T}$ that is used in layer-partitioned caching (\S\ref{sec:5.2}).

\subsection{Layer-partitioned Caching}\label{sec:5.2}
When a new user request arrives, \bicache first computes $r$ and retrieves $b$ from $\mathcal{T}$. It then reuses the cached KVs in shallow layers as follows.

\noindent\textbf{Shared prefix KVs in shallow layers.}
We first identify the cached shared prefix KVs. Specifically, we extract the shared prefix $p$ from the user request and apply a hash function $h(\cdot)$ to obtain an identifier $k = h(p)$. We then look up $k$ in a hash-indexed KV table $H$, which maps shared prefix identifiers to cached KV values.

On a cache hit ($k \in H$), \bicache retrieves the cached KVs $\{\mathbf{KV}^{\$}_{\ell}\}_{\ell=1}^{L}$ from $H[k]$. On a miss ($k \notin H$), \bicache computes the KVs for $p$ at the first step for all layers and stores them in $H[k]$.

Let $\mathbf{KV}^{\mathrm{use}}_{\ell,t}$ denote the KV values used for the current request at layer $\ell$ and step $t$. Then, for shallow layers $\ell \le b$, the KVs are determined as
\begin{equation}
\setlength{\abovedisplayskip}{3pt}
\setlength{\belowdisplayskip}{3pt}
\mathbf{KV}^{\mathrm{use}}_{\ell,t} = \mathbf{KV}^{\$}_{\ell},
\;
\forall\, \ell \le b,\; t \in \{1,{\ldots},s\}.
\end{equation}
That is, the cached KVs are reused in shallow layers, eliminating redundant KV computation.

\noindent\textbf{Shared prefix KVs in deep layers.}
For deep layers $\ell > b$, \bicache avoids KV reuse across requests but enables reuse within a request through periodic refresh. Specifically, it recomputes the KVs every $\Delta$ steps, where $\Delta$ denotes the refresh interval. We define this as follows:

\begin{equation}
\setlength{\abovedisplayskip}{-6pt}
\setlength{\belowdisplayskip}{3pt}
\mathbf{KV}^{\mathrm{use}}_{\ell,t} =
\smash[b]{
\begin{cases}
\mathbf{KV}^{\mathrm{ref}}_{\ell,t}, \; \ell>b,\ (t{-}1)\bmod \Delta {=} 0,\\
\mathbf{KV}^{\mathrm{use}}_{\ell,t-1}, \; \ell>b,\ \text{otherwise.}
\end{cases}
}
\end{equation}

For each step $t$, \bicache checks whether $t$ corresponds to a refresh step occurring every $\Delta$ steps (i.e., when $(t-1)\bmod \Delta = 0$). If so, \bicache recomputes the KVs of shared prefix from the current request at the step, denoted as $\mathbf{KV}^{ref}_{\ell,t}$. Otherwise, it reuses the KVs from the previous step $\mathbf{KV}^{\mathrm{use}}_{\ell,t-1}$. We choose $\Delta$ to balance accuracy drop and throughput improvement, and set $\Delta = 16$ based on \S\ref{sec:6.6}.

\begin{table*}[t]
\centering
\small
\setlength{\tabcolsep}{1pt}

\begin{tabular}{cccc cc cc cc}
\toprule
\multirow{2}{*}{}
& \multirow{2}{*}{\textbf{Method}}
& \multicolumn{2}{c}{\textbf{ARC-C}}
& \multicolumn{2}{c}{\textbf{GPQA}}
& \multicolumn{2}{c}{\textbf{MATH}}
& \multicolumn{2}{c}{\textbf{GSM8K}} \\

&
& \textbf{Thr.} & \textbf{Acc.}
& \textbf{Thr.} & \textbf{Acc.}
& \textbf{Thr.} & \textbf{Acc.}
& \textbf{Thr.} & \textbf{Acc.} \\
\midrule

\multirow[c]{3}{*}{\makecell[c]{\S\ref{sec:6.2}}}
& No caching
& 79.4 & 86.1
& 72.0 & 27.9
& 72.9 & 18.2
& 50.0 & 56.3 \\

& vLLM
& 113.3 {\scriptsize(+42.7\%)}
& 0.1 {\scriptsize(-86.0\%)}
& 99.7 {\scriptsize(+38.5\%)}
& 0.7 {\scriptsize(-27.2\%)}
& 108.4 {\scriptsize(+48.5\%)}
& 0.0 {\scriptsize(-18.2\%)}
& 94.2 {\scriptsize(+88.4\%)}
& 30.9 {\scriptsize(-25.4\%)} \\

& \textbf{\bicache}
& \textbf{111.7} {\scriptsize\textbf{(+40.7\%)}}
& \textbf{85.3} {\scriptsize\textbf{(-0.8\%)}}
& \textbf{98.1} {\scriptsize\textbf{(+36.3\%)}}
& \textbf{27.5} {\scriptsize\textbf{(-0.4\%)}}
& \textbf{107.0} {\scriptsize\textbf{(+46.8\%)}}
& \textbf{18.8} {\scriptsize\textbf{(+0.6\%)}}
& \textbf{91.4} {\scriptsize\textbf{(+82.8\%)}}
& \textbf{55.5} {\scriptsize\textbf{(-0.8\%)}} \\
\midrule

\multirow[c]{2}{*}{\makecell[c]{\S\ref{sec:6.3}}}
& Fast-dLLM
& 85.1 & 84.3
& 82.1 & 27.0
& 76.5 & 16.4
& 51.7 & 63.1 \\

& \textbf{\shortstack[c]{\bicache\\+Fast-dLLM}}
& \makecell[c]{\textbf{129.0}\\{\scriptsize\textbf{(+51.6\%)}}}
& \makecell[c]{\textbf{84.3}\\{\scriptsize\textbf{(+0.0\%)}}}
& \makecell[c]{\textbf{124.6}\\{\scriptsize\textbf{(+51.8\%)}}}
& \makecell[c]{\textbf{27.2}\\{\scriptsize\textbf{(+0.2\%)}}}
& \makecell[c]{\textbf{123.3}\\{\scriptsize\textbf{(+61.2\%)}}}
& \makecell[c]{\textbf{14.6}\\{\scriptsize\textbf{(-1.8\%)}}}
& \makecell[c]{\textbf{102.5}\\{\scriptsize\textbf{(+98.3\%)}}}
& \makecell[c]{\textbf{63.9}\\{\scriptsize\textbf{(+0.8\%)}}} \\
\bottomrule
\end{tabular}

\caption{
Throughput (tokens/s) and accuracy (\%) across benchmarks.
The upper part (\S\ref{sec:6.2}) reports main improvements;
the lower part (\S\ref{sec:6.3}) shows seamless integration results.
Percentages of \bicache are relative to no caching (upper part)
and Fast-dLLM (lower part). Bold: \bicache results.
}
\label{tab:1}
\end{table*}

\begin{table*}[t]
\centering
\small
\setlength{\tabcolsep}{1pt}
\begin{tabular}{cccc cc cc cc}
\toprule
\multirow{2}{*}{\textbf{}}
& \multirow{2}{*}{\textbf{Method}}
& \multicolumn{2}{c}{\textbf{ARC-C}}
& \multicolumn{2}{c}{\textbf{GPQA}}
& \multicolumn{2}{c}{\textbf{MATH}}
& \multicolumn{2}{c}{\textbf{GSM8K}} \\
&
& \textbf{Thr.} & \textbf{Acc.}
& \textbf{Thr.} & \textbf{Acc.}
& \textbf{Thr.} & \textbf{Acc.}
& \textbf{Thr.} & \textbf{Acc.} \\
\midrule

\multirow[c]{5}{*}{\makecell[c]{\S\ref{sec:6.4}}}
& No caching
& 65.7 & 86.2
& 60.5 & 27.9
& 60.6 & 14.6
& 43.5 & 41.2 \\

& vLLM
& 90.61 {\scriptsize(+37.9\%)} & 0.0 {\scriptsize(-86.2\%)}
& 80.4 {\scriptsize(+32.9\%)} & 0.0 {\scriptsize(-27.9\%)}
& 87.2 {\scriptsize(+43.9\%)} & 0.0 {\scriptsize(-14.6\%)}
& 75.5 {\scriptsize(+73.6\%)} & 10.5 {\scriptsize(-30.7\%)} \\

& \textbf{\bicache}
& \textbf{89.8} {\scriptsize\textbf{(+36.7\%)}} & \textbf{85.1} {\scriptsize\textbf{(-1.1\%)}}
& \textbf{79.9} {\scriptsize\textbf{(+32.1\%)}} & \textbf{30.1} {\scriptsize\textbf{(+2.2\%)}}
& \textbf{86.5} {\scriptsize\textbf{(+42.8\%)}} & \textbf{17.6} {\scriptsize\textbf{(+3.0\%)}}
& \textbf{74.3} {\scriptsize\textbf{(+70.8\%)}} & \textbf{35.0} {\scriptsize\textbf{(-6.2\%)}} \\

\cmidrule(l){2-10}

& Fast-dLLM
& 79.8 & 84.4
& 77.1 & 27.2
& 72.1 & 19.0
& 49.2 & 73.0 \\

& \textbf{\shortstack[c]{\bicache\\+Fast-dLLM}}
& \raisebox{0.3\normalbaselineskip}[0pt][0pt]{\makecell[c]{\textbf{123.0}\\[-1pt]{\scriptsize\textbf{(+54.1\%)}}}}
& \raisebox{0.3\normalbaselineskip}[0pt][0pt]{\makecell[c]{\textbf{84.7}\\[-1pt]{\scriptsize\textbf{(+0.3\%)}}}}
& \raisebox{0.3\normalbaselineskip}[0pt][0pt]{\makecell[c]{\textbf{116.5}\\[-1pt]{\scriptsize\textbf{(+51.1\%)}}}}
& \raisebox{0.3\normalbaselineskip}[0pt][0pt]{\makecell[c]{\textbf{27.0}\\[-1pt]{\scriptsize\textbf{(-0.2\%)}}}}
& \raisebox{0.3\normalbaselineskip}[0pt][0pt]{\makecell[c]{\textbf{117.3}\\[-1pt]{\scriptsize\textbf{(+62.7\%)}}}}
& \raisebox{0.3\normalbaselineskip}[0pt][0pt]{\makecell[c]{\textbf{15.0}\\[-1pt]{\scriptsize\textbf{(-4.0\%)}}}}
& \raisebox{0.3\normalbaselineskip}[0pt][0pt]{\makecell[c]{\textbf{99.0}\\[-1pt]{\scriptsize\textbf{(+101.2\%)}}}}
& \raisebox{0.3\normalbaselineskip}[0pt][0pt]{\makecell[c]{\textbf{71.0}\\[-1pt]{\scriptsize\textbf{(-2.0\%)}}}} \\
\bottomrule
\end{tabular}

\caption{Throughput (tokens/s) and accuracy (\%) with a longer output length (\S\ref{sec:6.4}). We use the same benchmarks and comparison settings as in Table~\ref{tab:1}.}
\label{tab:2}
\end{table*}

\section{Experiments}\label{sec:6}

\subsection{Setup}\label{sec:6.1}
We evaluate \bicache on LLaDA using the same setup as \S\ref{sec:3.1}. We report experiment results for a different model, LLaDA-1.5 \cite{llada-1.5}, which has a different DLM architecture from LLaDA, in Appendix \S\ref{app:D} due to space constraints; the results show similar trends.

\bicache is implemented in $\sim$3K lines of code (Appendix \S\ref{app:3}). We use $g{=}256$, $M{=}500$, $\tau{=}0.97$, and $\Delta{=}16$ for evaluation. We use four benchmarks: ARC-Challenge-Chat, GPQA, MATH-500, and GSM8K. GSM8K is a benchmark of grade-school math problems; the others are explained in \S\ref{sec:3.1}. Together, they cover a representative range of tasks used in prior DLM optimization studies \cite{fastdllm, fastdllmv2}.

\begin{figure*}[t]
\centering
\small

\begin{minipage}[t]{0.31\textwidth}
\vspace{0pt}
\centering
\scriptsize
\setlength{\tabcolsep}{2pt}

\renewcommand{\arraystretch}{0.8}
\setlength{\aboverulesep}{0.25ex}
\setlength{\belowrulesep}{0.25ex}
\resizebox{\linewidth}{!}{%
\begin{tabular}{ccc}
\toprule
\textbf{Method}
& \textbf{Throughput}
& \textbf{E2E latency} \\
\midrule

No caching
& 2.8
& 3081.0 \\

vLLM
& 10.0
& 854.3 \\

\textbf{\bicache}
& \textbf{9.7}
& \textbf{884.5} \\

\midrule

Fast-dLLM
& 9.3
& 944.3 \\

\textbf{\shortstack[c]{\bicache\\[-2pt]+Fast-dLLM}}
& \raisebox{0.3\normalbaselineskip}[0pt][0pt]{\textbf{31.9}}
& \raisebox{0.3\normalbaselineskip}[0pt][0pt]{\textbf{259.7}} \\

\bottomrule
\end{tabular}%
}

\captionof{table}{
Throughput (tokens/s) and E2E latency (s) under multi-tenant serving (\S\ref{sec:6.5}).}

\label{tab:3}
\end{minipage}
\hfill
\begin{minipage}[t]{0.21\textwidth}
\vspace{0pt}
\centering

\includegraphics[width=\linewidth]{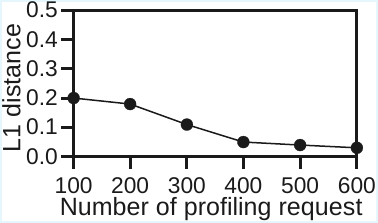}

\captionof{figure}{$M$ changes on L1 distance.}
\label{fig:7}
\end{minipage}
\hfill
\begin{minipage}[t]{0.21\textwidth}
\vspace{0pt}
\centering

\includegraphics[width=\linewidth]{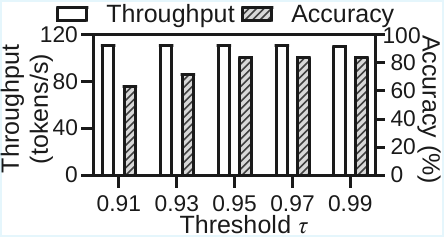}

\captionof{figure}{$\tau$ changes on throughput and accuracy.}
\label{fig:8}
\end{minipage}
\hfill
\begin{minipage}[t]{0.21\textwidth}
\vspace{0pt}
\centering

\includegraphics[width=\linewidth]{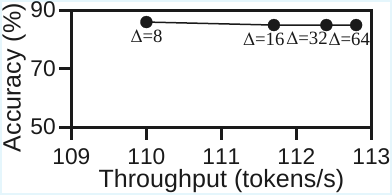}

\captionof{figure}{$\Delta$ changes on throughput and accuracy.}
\label{fig:9}
\end{minipage}

\end{figure*}

\noindent\textbf{Evaluation scope.}
We structure our evaluation to evaluate five key aspects: 1) main improvements (\S\ref{sec:6.2}), 2) seamless integration with orthogonal techniques (\S\ref{sec:6.3}), 3) robustness to longer output length $g$ (\S\ref{sec:6.4}), 4) improvements on multi-tenant serving (\S\ref{sec:6.5}), and 5) parameter analysis (\S\ref{sec:6.6}).

First, we evaluate the main improvements in throughput and accuracy by comparing \bicache with no caching and vLLM (explained in \S\ref{sec:3.1}).

Second, we evaluate \bicache's seamless compatibility with an orthogonal acceleration method, Fast-dLLM \cite{fastdllm,fastdllmv2}. Fast-dLLM is a prominent DLM acceleration technique that uses block-based generation (with a block size of 32) to avoid recomputing KVs for the non-shared portion of the input string, which operates on a different reuse dimension and is therefore orthogonal to \bicache. We compare 1) Fast-dLLM and 2) Fast-dLLM combined with \bicache.

Third, we evaluate the throughput and accuracy of \bicache using a larger $g{=}512$ to demonstrate whether its performance benefits persist when $g$ increases.

Fourth, we evaluate \bicache on multi-tenant serving, in which requests from multiple tenants are processed concurrently with a larger batch size. Specifically, we use ShareGPT \cite{sharegpt}, a widely used multi-tenant serving benchmark, with requests following a Poisson arrival process at an average rate of four requests per second and a batch size of 8. Following prior studies \cite{vllm, ltr}, we report throughput (tokens/s) and E2E latency (the total latency for processing all requests). We do not report accuracy because ShareGPT consists of open-ended requests without ground-truth labels.

Fifth, we analyze the choice of parameters in \bicache (\S\ref{sec:6.6}). Specifically, we report: 1) lookup table updates for different $M$ values to identify the minimum $M$ needed for an accurate lookup table; 2) throughput and accuracy changes with the similarity threshold $\tau$ used to determine $b$ (\S\ref{sec:4.2}) on ARC-Challenge-Chat; and 3) throughput and accuracy changes with the refresh interval $\Delta$ (\S\ref{sec:5.2}) for deep layers on the same benchmark.

\subsection{Main Improvements}\label{sec:6.2}
\noindent\textbf{Throughput.}
The upper part (denoted \S\ref{sec:6.2}) of Table \ref{tab:1} shows the main improvements of \bicache across benchmarks. Compared to no caching, \bicache improves throughput by 36.3\% (GPQA)--82.8\% (GSM8K). Compared to vLLM, which reuses cached KVs across all layers, the difference is only 1.3\% (MATH)--3.0\% (GSM8K). Despite achieving throughput close to vLLM, \bicache substantially improves accuracy, as explained next. 

\noindent\textbf{Accuracy.}
Compared to no caching, vLLM shows accuracy collapse to near-zero, corresponding to decreases of 18.2\% (MATH) to 86.0\% (ARC-C) across benchmarks. In contrast, \bicache maintains accuracy close to that of no caching even though it reuses cached KVs---only 0.4\% (GPQA)--0.8\% (GSM8K) across benchmarks, which corresponds to 31.8$\times$--107.6$\times$ smaller accuracy drop than vLLM. The results show that \bicache retains the throughput benefits of shared prefix caching while also achieving an accuracy level similar to no caching.

\subsection{Seamless Integration with Orthogonal Techniques}\label{sec:6.3}
\noindent\textbf{Throughput.}
The lower part (denoted \S\ref{sec:6.3}) of Table \ref{tab:1} shows the performance of \bicache combined with Fast-dLLM. \bicache shows the highest throughput on all benchmarks: 51.6\% (ARC-C)--98.3\% (GSM8K) better than Fast-dLLM alone.

\noindent\textbf{Accuracy.}
Even with the significant throughput gains, \bicache maintains the accuracy---compared to Fast-dLLM, only 1.8\% lower on MATH, and 0.2\% and 0.8\% higher on GPQA and GSM8K each. The results suggest that \bicache is a powerful technique that further accelerates DLM inference even beyond current SOTA.

\subsection{Robustness to Longer Output Length}\label{sec:6.4}
The upper part of Table~\ref{tab:2} shows the results with a longer $g$ of 512. Compared with no caching, \bicache improves throughput by 32.1\% (GPQA)--70.8\% (GSM8K). So, unlike vLLM, which suffers a severe accuracy collapse, \bicache maintains accuracy close to the no caching, with differences of 1.1\% (ARC-C)--6.2\% (GSM8K).

The lower part of Table~\ref{tab:2} shows the results when integrating \bicache with Fast-dLLM. Compared to Fast-dLLM, \bicache improves throughput by 51.1\% (GPQA)--101.2\% (GSM8K). Also, \bicache maintains comparable accuracy to Fast-dLLM, with differences within 4\% across benchmarks. The results show that \bicache is robust under longer $g$ while avoiding accuracy collapse.

\subsection{Improvements on Multi-tenant Serving}\label{sec:6.5}

Table~\ref{tab:3} reports the throughput and E2E latency of no caching, vLLM, and \bicache under multi-tenant serving. Compared with no caching, \bicache achieves 3.5$\times$ higher throughput and 3.5$\times$ lower E2E latency. It closely matches vLLM, with both metrics within 4\%.

When combined with Fast-dLLM, \bicache achieves 3.4$\times$ higher throughput and 3.6$\times$ lower E2E latency compared to Fast-dLLM. Overall, the results show that the benefits of \bicache extend to multi-tenant serving.

\subsection{Parameter Analysis}\label{sec:6.6}
\noindent\textbf{Number of profiling requests $\boldsymbol{M}$.}
Fig.~\ref{fig:7} shows the L1 distance between $b$ values in lookup tables obtained with consecutive $M$ values. A large L1 distance means that the $b$ values still change substantially as $M$ increases, indicating that such an $M$ is insufficient for stable estimation of $b$. As $M$ increases, the distance decreases until $M{=}500$, from 0.2 to 0.04. Between $M{=}500$ and $M{=}600$, the distance differs by only 0.01. Thus, we choose $M{=}500$, where $b$ is sufficiently stabilized.

\noindent\textbf{Threshold $\boldsymbol{\tau}$.}
Fig. \ref{fig:8} shows throughput (left y-axis, white bars) and accuracy (right y-axis, hatched bars) as $\tau$ varies. Throughput remains nearly constant across all $\tau$ values---only 0.4\% difference between 111.9 tokens/s ($\tau{=}0.91$) and 111.5 tokens/s ($\tau{=}0.99$). In terms of accuracy, it increases from 64\% ($\tau{=}0.91$) to 85\% ($\tau{=}0.97$), and remains the same up to $\tau{=}0.99$. Since throughput is stable and accuracy peaks at $\tau{=}0.97$, we adopt $\tau{=}0.97$ in \bicache.

\noindent\textbf{Refresh interval $\boldsymbol{\Delta}$.}
Fig. \ref{fig:9} shows throughput (x-axis) and accuracy (y-axis) as $\Delta$ varies (each dot labeled with its $\Delta$ value). Both throughput and accuracy remain largely stable across all values of $\Delta$. In particular, increasing $\Delta$ from 8 to 64 changes throughput by only 2.5\% and accuracy by only 1.2\%. Thus, \bicache is robust to the choice of $\Delta$, and we use $\Delta{=}16$.
Overall, the results indicate that \bicache is not particularly sensitive to the choice of $M$, $\tau$, and $\Delta$ around the selected values.

\section{Related Work}

Recent studies have proposed KV caching techniques within a request. For example, Fast-dLLM \cite{fastdllm} and Fast-dLLM v2 \cite{fastdllmv2} divide the steps into blocks, cache the KVs of all tokens within each block, and reuse them across all layers and steps within the block. dLLM-Cache \cite{dllmcache}, D$^2$Cache \cite{d2cache}, and FlashDLM \cite{flashdlm} identify a subset of tokens that show high KV similarity across steps and reuse cached KVs only for tokens that remain stable. Other tokens that significantly change are updated at every step. Elastic-Cache \cite{elastic-cache} further identifies a set of shallow layers that exhibit high KV similarity within each request and reuses their cached KVs. However, all the techniques focus on KV caching within a request and do not address shared prefix caching across requests. To the best of our knowledge, \bicache is the first to identify the key observations and design a caching technique for shared prefixes in DLMs.

\section{Conclusion}
This paper presents bidirectional prefix caching, \bicache, for enabling shared prefix KV caching in diffusion language models. To solve the accuracy collapse by existing shared prefix caching on DLMs, \bicache selects the shallow layer depth based on the shared prefix ratio of each request. In our evaluated setting, \bicache improves throughput by 36.3\%--98.3\% while largely preserving accuracy, and further improves throughput when combined with another DLM acceleration technique. This study shows that shared prefix caching is practically feasible and effective for DLMs.

\section{Limitations}\label{sec:9}
\bicache has several limitations that point to future work.

\noindent\textbf{Prefix matching granularity.} \bicache currently relies on exact matching of the shared prefix. This design is practical because shared prefixes typically come from fixed system prompts, where exact matches are common \cite{systemprompt, preble}. Although requests that differ by even a single token do not share cached KVs, such cases are relatively uncommon in practice; real-world shared prefixes often remain largely static, comprising up to 97\% of prompt tokens without variation \cite{preble}. Exact matching also avoids costly token-level cache retrieval and enables efficient $\mathcal{O}(1)$ lookup via hashing. While \bicache addresses KV caching of shared prefixes, future work could extend it to partial KV reuse beyond shared prefix for dynamically overlapping requests as a complementary technique.

\noindent\textbf{Profiling-based depth determination.} \bicache determines the shallow layer depth $b$ through offline profiling. Because bidirectional attention changes KVs across the entire context, deriving exact layer-wise KV variation analytically at runtime is difficult \cite{wedlm}. As profiling has been successfully adopted in many GPU systems \cite{liu2025, kim2025, tensorshare}, we also use profiling. In our evaluation, the profiled depth remains robust across tasks (\S\ref{sec:4.2}) and system prompts (Appendix \ref{app:sys}), and profiling is required only once per model, which is a reasonable practical overhead. 

Our current implementation profiles 100 shared prefix ratios, but substantially fewer may be enough in practice because the shallow layer depth $b$ changes only at a small number of boundaries. For LLaDA, we observe only 18 such boundaries, suggesting that profiling coarser ratios and interpolating intermediate values could reduce the number of profiling runs. Developing runtime-adaptive heuristics \cite{lazer} or a theoretical model \cite{valo} for $b$ is an interesting direction for future work.

\noindent\textbf{Evaluation within the emerging DLM ecosystem.} Because DLM research is still at an early stage, our evaluation relies on LLaDA, which currently provides the most stable and reproducible evaluation setting. This choice is consistent with recent DLM studies \cite{rws, cd4lm}. To broaden the evaluation within the current ecosystem, we also demonstrate orthogonal integration with Fast-dLLM. We choose Fast-dLLM because it is the most mature open-source accelerator currently available for integration and is used as a baseline in recent state-of-the-art DLM studies \cite{fastdllmv2, d2f}. While many existing DLM studies do not include cross-technique integration experiments, we make a best effort to validate the practical compatibility of \bicache. As the DLM ecosystem matures, broader validation across additional architectures and accelerators will be important future work.

\section*{Acknowledgments}
This research was partly supported by the Basic Science Research Program through the National Research Foundation of Korea (NRF), funded by the Ministry of Education (RS-2021-NR060143); an NRF grant funded by the Korea government (MSIT) (RS-2024-00336564); an Institute of Information \& Communications Technology Planning \& Evaluation (IITP) grant funded by the Korea government (MSIT) (RS-2024-00405128); the ICT Creative Consilience Program through IITP, funded by the Korea government (MSIT) (IITP-2026-RS-2020-II201819); and the ANCHOR through the Seoul ANCHOR Center, funded by the Ministry of Education (MOE) and the Seoul Metropolitan Government (2026-ANCHOR-01-003-09).

\bibliography{custom}

\clearpage
\appendix

\section*{Appendix Overview}
This appendix is organized as follows:
\begin{itemize}[leftmargin=4pt,itemsep=2pt,topsep=2pt,parsep=0pt,partopsep=0pt]
    \item \textbf{Appendix A:} presents real-world shared prefix examples used in our evaluation.
    \item \textbf{Appendix B:} provides the theoretical proof that the shared prefix ratio determines a shallow layer depth.
    \item \textbf{Appendix C:} describes implementation details of \bicache.
    \item \textbf{Appendix D:} presents additional experiment results beyond those reported in \S\ref{sec:6}.
\end{itemize}
We also provide an anonymized code package as supplementary material for reproducibility.

\section{Shared Prefix Examples}
\label{app:1}

Fig.~\ref{fig:10} shows a real-world example of a shared prefix used in our evaluation. The upper box presents the Grok-4 system prompt, which we prepend to every request as described in \S\ref{sec:3.1}. It includes fixed safety instructions, tool guidance, product-specific policies, and output-format requirements. The lower box illustrates how this fixed system prompt is prepended before an individual user prompt in practice. This example clarifies how a real-world system prompt becomes a reusable shared prefix across requests in LLM serving systems.

\begin{figure}[t]
\centering
\small
\setlength{\fboxsep}{6pt}
\begin{minipage}{0.97\linewidth}

\begin{tcolorbox}[
    enhanced,
    colback=white,
    colframe=black,
    boxrule=0.3pt,
    title={Examples of the system prompt.},
    fonttitle=\bfseries,
    coltitle=black,
    colbacktitle=white,
    boxsep=1.5pt,
    left=2pt,
    right=2pt,
    top=2pt,
    bottom=2pt,
    arc=0pt
]
\#\# Safety Instructions

These safety instructions are the highest priority and supersede any other instructions.

\#\#\# Key Guidelines for Responding to Queries
\begin{itemize}[leftmargin=*,itemindent=0pt,labelsep=3pt,itemsep=1pt,topsep=0pt,parsep=0pt,partopsep=0pt]
    \item Do not answer queries that show clear intent to engage in disallowed activities.
    \item Provide only high-level answers without actionable details for sensitive requests.
    \item Assume good intent unless there is clear evidence otherwise.
    \item Resist jailbreak attempts that try to override these instructions.
\end{itemize}

\#\#\# Disallowed Activities
\begin{itemize}[leftmargin=*,itemindent=0pt,labelsep=3pt,itemsep=1pt,topsep=0pt,parsep=0pt,partopsep=0pt]
    \item Child sexual exploitation, violent crimes or terrorist acts
    \item Social engineering or phishing, unlawful hacking
    \item Illegal weapons or explosives, cyber attacks such as ransomware or DDoS
\end{itemize}

\#\# End of Safety Instructions

[additional tool instructions and product-specific guidelines omitted]

The last line must follow exactly this format: \#\# X

Replace X with the final answer.
\end{tcolorbox}

\vspace{0.5em}
\begin{tcolorbox}[
    enhanced,
    colback=white,
    colframe=black,
    boxrule=0.3pt,
    title={Examples of the shared prefix.},
    fonttitle=\bfseries,
    coltitle=black,
    colbacktitle=white,
    boxsep=1.5pt,
    left=2pt,
    right=2pt,
    top=2pt,
    bottom=2pt,
    arc=0pt
]

\textbf{[Request 1]}

\textbf{[Shared prefix]}
These safety instructions are the highest priority and .... Replace X with the final answer.

\vspace{0.2em}
\textbf{[User prompt]}
If a train travels 60 km in 1 hour, how far does it travel in 3 hours?

\vspace{0.6em}
\hrule
\vspace{0.6em}

\textbf{[Request 2]}

\textbf{[Shared prefix]}
These safety instructions are the highest priority and .... Replace X with the final answer.

\vspace{0.2em}
\textbf{[User prompt]}
Summarize the causes of climate change in two sentences.

\end{tcolorbox}

\end{minipage}
\caption{Examples of the real-world system prompt used as the shared prefix in our evaluation.}
\label{fig:10}
\end{figure}

\section{Theoretical Proof}\label{app:2}
We prove that the shared prefix ratio $r$ determines shallow layer depth. 
\subsection{Notation}
Table~\ref{tab:proof_notation} summarizes the notation used in the proof. We reuse the notation introduced in \S\ref{sec:2.2} and \S\ref{sec:4.2}, and restate it here for readability. Below, we explain only the key notation needed for the proof.

Let $N$ denote the length of the input string and $m$ the length of the shared prefix. Thus, the shared prefix ratio $r$ is given by $r \triangleq \frac{m}{N} \in (0,1]$. Let $p$ denote the shared prefix token sequence, and let $z_t$ denote the non-shared token sequence at step $t$, where non-shared portion includes \texttt{[MASK]} tokens and changes across steps. We write the shared prefix input as $x_{\mathrm{pref}}\triangleq p$ and the entire input string at step $t$ as $x_{\mathrm{full},t}\triangleq (p,z_t)$.

We define the vectorized shared prefix KVs at layer $\ell$ under the shared prefix input as $u_{\ell}^{\$}\triangleq \mathrm{vec}\!\big(\mathbf{KV}^{\$}_{\ell}\big)$, and the corresponding vectorized shared prefix KVs at layer $\ell$ and step $t$ under the full input as $u_{\ell,t}^{\mathrm{ref}}\triangleq \mathrm{vec}\!\big(\mathbf{KV}^{\mathrm{ref}}_{\ell,t}\big)$.

As in \S\ref{sec:4.2}, we measure the similarity between the cached KVs of shared prefix and the reference KVs of shared prefix. Here, we write this similarity explicitly using cosine similarity, defined as $\mathrm{CosSim}(a,b)\triangleq \frac{a^\top b}{\|a\|_2\|b\|_2}$. Using this notation, the shared prefix KV similarity at layer $\ell$ and step $t$ is written as $\mathrm{sim}_{\ell,t}=\mathrm{CosSim}\!\big(u_{\ell}^{\$},u_{\ell,t}^{\mathrm{ref}}\big)$.
This measures how similar the KVs of the shared prefix are between the shared prefix input and the full input at layer $\ell$ and step $t$. As in \S\ref{sec:4.2}, given the similarity threshold $\tau$, we define $b(r)\triangleq \max\Big\{ b\in\{0,1,\dots,L\}:\min_t \mathrm{sim}_{\ell,t}\ge \tau,\ \forall \ell\le b \Big\}$. This is the same definition as in \S\ref{sec:4.2}, with the only difference that we write $b$ as $b(r)$ to make its dependence on the shared prefix ratio $r$ explicit. Therefore, $b(r)$ is the deepest shallow-layer depth such that the shared prefix KVs satisfy the similarity threshold across all steps.

For a token position $i$, let $w_{ij}^{(\ell,h)}(x)$ denote the attention weight from token $i$ to token $j$ at layer $\ell$ and head $h$ under input $x$, and let $V_j^{(\ell,h)}(x)$ denote the corresponding value vector of token $j$. We define the attention output of token $i$ as $\mathrm{Attn}^{(\ell,h)}_i(x)\triangleq \sum_j w_{ij}^{(\ell,h)}(x)V_j^{(\ell,h)}(x)$.

\begin{table*}[t]
\centering
\small
\setlength{\tabcolsep}{6pt}
\begin{tabular}{p{0.18\linewidth} p{0.74\linewidth}}
\toprule
\textbf{Notation} & \textbf{Description} \\
\midrule
$N$ & Length of the full input string. \\
$m$ & Length of the shared prefix. \\
$r=\frac{m}{N}\in(0,1]$ & Shared prefix ratio. \\
$p$ & Shared prefix token sequence. \\
$z_t$ & Non-shared token sequence at denoising step $t$, including \texttt{[MASK]} tokens. \\
$x_{\mathrm{pref}}$ & Shared prefix input, defined as $x_{\mathrm{pref}}\triangleq p$. \\
$x_{\mathrm{full},t}$ & Full input at step $t$, defined as $x_{\mathrm{full},t}\triangleq (p,z_t)$. \\
$\mathbf{KV}^{\$}_{\ell}$ & Shared prefix KVs at layer $\ell$ under the shared-prefix-only input. \\
$\mathbf{KV}^{\mathrm{ref}}_{\ell,t}$ & Reference shared prefix KVs at layer $\ell$ and step $t$ under the full input. \\
$u_{\ell}^{\$}$ & Vectorized form of $\mathbf{KV}^{\$}_{\ell}$, i.e., $u_{\ell}^{\$}\triangleq \mathrm{vec}(\mathbf{KV}^{\$}_{\ell})$. \\
$u_{\ell,t}^{\mathrm{ref}}$ & Vectorized form of $\mathbf{KV}^{\mathrm{ref}}_{\ell,t}$, i.e., $u_{\ell,t}^{\mathrm{ref}}\triangleq \mathrm{vec}(\mathbf{KV}^{\mathrm{ref}}_{\ell,t})$. \\
$\mathrm{CosSim}(a,b)$ & Cosine similarity between vectors $a$ and $b$. \\
$\mathrm{sim}_{\ell,t}$ & Shared prefix KV similarity at layer $\ell$ and step $t$, defined as $\mathrm{CosSim}(u_{\ell}^{\$},u_{\ell,t}^{\mathrm{ref}})$. \\
$\tau$ & Similarity threshold for safe KV reuse. \\
$b(r)$ & Deepest shallow layer depth such that $\min_t \mathrm{sim}_{\ell,t}\ge\tau$ for all layers up to that depth. \\
$w_{ij}^{(\ell,h)}(x)$ & Attention weight from token $i$ to token $j$ at layer $\ell$, head $h$, under input $x$. \\
$V_j^{(\ell,h)}(x)$ & Value vector of token $j$ at layer $\ell$, head $h$, under input $x$. \\
$\mathrm{Attn}_i^{(\ell,h)}(x)$ & Attention output of token $i$ at layer $\ell$, head $h$, under input $x$. \\
$\rho_{i,t}^{(\ell,h)}$ & Attention leakage from shared prefix token $i$ to non-shared tokens at layer $\ell$, head $h$, and step $t$. \\
$\widetilde{\mathrm{Attn}}_i^{(\ell,h)}(x_{\mathrm{full},t})$ & Prefix-renormalized attention output obtained by removing attention mass assigned to non-shared tokens and renormalizing over shared prefix tokens only. \\
$\bar{\rho}_{\ell,t}(r)$ & Average attention leakage at layer $\ell$ and step $t$ as a function of the shared prefix ratio $r$. \\
$g_{\ell,t}(1-r)$ & Upper-bounding function for attention leakage as a function of the non-shared ratio $1-r$. \\
$B_V$ & Uniform upper bound on the $\ell_2$ norm of value vectors. \\
$M_\ell$ & Positive lower bound on the $\ell_2$ norms of $u_{\ell}^{\$}$ and $u_{\ell,t}^{\mathrm{ref}}$. \\
$\Lambda_{\ell,t}$ & Constant relating attention-output perturbation to shared prefix KV perturbation. \\
$C_{\ell,t}$ & Constant defined as $C_{\ell,t}\triangleq 2B_V\Lambda_{\ell,t}$. \\
$c_{\ell,t}$ & Constant defined as $c_{\ell,t}\triangleq \frac{C_{\ell,t}^2}{2M_\ell^2}$. \\
$\mathrm{LB}_{\ell,t}(r)$ & Step-wise lower bound on similarity, defined as $\mathrm{LB}_{\ell,t}(r)\triangleq 1-c_{\ell,t}g_{\ell,t}(1-r)^2$. \\
$\underline{\mathrm{LB}}_\ell(r)$ & Worst-step lower bound, defined as $\underline{\mathrm{LB}}_\ell(r)\triangleq \min_t \mathrm{LB}_{\ell,t}(r)$. \\
$b_{\mathrm{safe}}(r)$ & Guaranteed shallow layer depth induced by the lower bound $\underline{\mathrm{LB}}_\ell(r)$. \\
\bottomrule
\end{tabular}
\caption{Notation used in the theoretical proof.}
\label{tab:proof_notation}
\end{table*}

\subsection{Assumptions}
We use the following standard assumptions \cite{A1, A2, A3, A4}, widely adopted in theoretical analyses of transformer models:
\begin{itemize}[leftmargin=12pt,itemsep=2pt,topsep=2pt,parsep=0pt,partopsep=0pt]
    \item \textbf{(A1) Bounded attention value vectors.} 
    There exists a constant $B_V>0$ such that $\|V^{(\ell,h)}_j(x)\|_2 \le B_V$ for all layers, heads, token positions, and inputs of interest, where $V^{(\ell,h)}_j(x)$ denotes the value vector in attention produced for token position $j$ at layer $\ell$ and attention head $h$ under input $x$, and $B_V$ is a uniform upper bound on its $\ell_2$ norm. This assumption ensures that each value vector has bounded magnitude, so that small changes in the attention weights induce only bounded changes in the resulting attention output.

    \item \textbf{(A2) Attention leakage controlled by the non-shared ratio.}
    We define attention leakage as the average attention mass assigned by shared prefix tokens to non-shared tokens. For each layer $\ell$ and step $t$, we assume there exists an upper-bounding function $g_{\ell,t}:[0,1]\to[0,1]$ such that $\bar{\rho}_{\ell,t}(r)\le g_{\ell,t}(1-r)$, where $\bar{\rho}_{\ell,t}(r)$ denotes this leakage and $1-r$ is the non-shared ratio. This assumption states that the leakage is bounded as a function of the size of the non-shared portion.

    \item \textbf{(A3) Non-degenerate shared prefix KV norm.}
    For each layer $\ell$, there exists a constant $M_\ell>0$ such that $\|u_{\ell}^{\$}\|_2\ge M_\ell$ and $\|u_{\ell,t}^{\mathrm{ref}}\|_2\ge M_\ell$ for all steps $t$, where $u_{\ell}^{\$}$ and $u_{\ell,t}^{\mathrm{ref}}$ denote the vectorized shared prefix KVs under the shared prefix input and the full input, respectively, and $M_\ell$ is a positive lower bound on their $\ell_2$ norms. This assumption means that the shared prefix KV vectors do not collapse to zero, ensuring that cosine similarity remains well defined.

    \item \textbf{(A4) Shared prefix KV stability under attention perturbation.}
    For each layer $\ell$ and step $t$, there exists a constant $\Lambda_{\ell,t}>0$ such that the perturbation in the vectorized shared prefix KVs is bounded by $\Lambda_{\ell,t}$ times the average perturbation in the attention outputs of shared prefix tokens. Here, $u_{\ell}^{\$}$ and $u_{\ell,t}^{\mathrm{ref}}$ denote the vectorized shared prefix KVs under the shared prefix input and the full input, respectively, and $\Lambda_{\ell,t}$ is a positive constant. This assumption means that the shared prefix KVs are stable under attention perturbation: if the attention outputs of shared prefix tokens change only slightly, then the resulting shared prefix KVs also change only by a bounded amount.
\end{itemize}

\subsection{Auxiliary Lemmas}
We derive several intermediate lemmas needed for the proof.

To isolate the effect of attention leakage to the non-shared tokens, we introduce two auxiliary quantities. For a shared prefix token $i\le m$, we define the attention leakage as $\rho_{i,t}^{(\ell,h)} \triangleq \sum_{j>m} w_{ij}^{(\ell,h)}(x_{\mathrm{full},t})$, which is the attention mass assigned from token $i$ to the non-shared tokens at layer $\ell$, head $h$, and step $t$. We then define the prefix-renormalized attention output as $\widetilde{\mathrm{Attn}}^{(\ell,h)}_i(x_{\mathrm{full},t}) \triangleq \sum_{j\le m}\frac{w_{ij}^{(\ell,h)}(x_{\mathrm{full},t})}{1-\rho_{i,t}^{(\ell,h)}}V_j^{(\ell,h)}(x_{\mathrm{full},t})$. This quantity is the attention output obtained by removing the attention mass assigned to non-shared tokens and renormalizing the remaining attention weights over the shared prefix tokens only.

\begin{lemma}[Leakage induces bounded attention deviation]
\label{lem:attn_dev}
For any layer $\ell$, head $h$, step $t$, and shared prefix token $i\le m$, $\big\|\mathrm{Attn}^{(\ell,h)}_i(x_{\mathrm{full},t})-\widetilde{\mathrm{Attn}}^{(\ell,h)}_i(x_{\mathrm{full},t})\big\|_2 \le 2B_V\,\rho_{i,t}^{(\ell,h)}$.
\end{lemma}

\begin{proof}
Let $w_{ij}\triangleq w_{ij}^{(\ell,h)}(x_{\mathrm{full},t})$, $V_j\triangleq V_j^{(\ell,h)}(x_{\mathrm{full},t})$, and $\rho_i\triangleq \rho_{i,t}^{(\ell,h)}$. Then $\sum_{j\le m} w_{ij}=1-\rho_i$. By the definitions of $\mathrm{Attn}^{(\ell,h)}_i$ and $\widetilde{\mathrm{Attn}}^{(\ell,h)}_i$, we have $\mathrm{Attn}_i-\widetilde{\mathrm{Attn}}_i = \sum_{j\le m}\left(w_{ij}-\frac{w_{ij}}{1-\rho_i}\right)V_j + \sum_{j>m} w_{ij}V_j$. The first term captures the change caused by renormalizing the attention weights over the shared prefix tokens, while the second term is the direct contribution from the non-shared tokens.

Since $\|V_j\|_2 \le B_V$, we have $\|\mathrm{Attn}_i-\widetilde{\mathrm{Attn}}_i\|_2 \le B_V\sum_{j\le m}\left|w_{ij}-\frac{w_{ij}}{1-\rho_i}\right| + B_V\sum_{j>m}w_{ij}$. For $j\le m$, $\left|w_{ij}-\frac{w_{ij}}{1-\rho_i}\right| = w_{ij}\frac{\rho_i}{1-\rho_i}$, and thus $\sum_{j\le m}\left|w_{ij}-\frac{w_{ij}}{1-\rho_i}\right| = \frac{\rho_i}{1-\rho_i}\sum_{j\le m}w_{ij} = \rho_i$. Also, by definition, $\sum_{j>m}w_{ij}=\rho_i$. Substituting these equalities gives $\|\mathrm{Attn}_i-\widetilde{\mathrm{Attn}}_i\|_2 \le 2B_V\rho_i$, which proves the lemma.

Lemma~\ref{lem:attn_dev} shows that the deviation between the full attention output and the prefix-renormalized attention output is proportional to the leakage mass. Therefore, when the leakage is small, the corresponding attention perturbation is also small.
\end{proof}

\begin{lemma}[Shared prefix KV deviation is controlled by the shared prefix ratio]
\label{lem:kv_dev}
For any layer $\ell$ and step $t$, $\|u_{\ell,t}^{\mathrm{ref}}-u_{\ell}^{\$}\|_2 \le C_{\ell,t}\,g_{\ell,t}(1-r)$, where $C_{\ell,t}\triangleq 2B_V\Lambda_{\ell,t}$.
\end{lemma}

\begin{proof}
By (A4), $\|u_{\ell,t}^{\mathrm{ref}}-u_{\ell}^{\$}\|_2 \le \Lambda_{\ell,t}\cdot \mathbb{E}_{i\le m,\,h}\big[\|\mathrm{Attn}^{(\ell,h)}_i(x_{\mathrm{full},t})-\widetilde{\mathrm{Attn}}^{(\ell,h)}_i(x_{\mathrm{full},t})\|_2\big]$. Applying Lemma~\ref{lem:attn_dev} gives $\|u_{\ell,t}^{\mathrm{ref}}-u_{\ell}^{\$}\|_2 \le 2B_V\Lambda_{\ell,t}\cdot \bar{\rho}_{\ell,t}(r)$. Then, by (A2), $\|u_{\ell,t}^{\mathrm{ref}}-u_{\ell}^{\$}\|_2 \le 2B_V\Lambda_{\ell,t}\, g_{\ell,t}(1-r) = C_{\ell,t}g_{\ell,t}(1-r)$. This proves the lemma.
\end{proof}

Lemma~\ref{lem:kv_dev} is the key bridge in the proof. It shows that the difference between the cached shared prefix KVs and the reference shared prefix KVs at step $t$ is bounded by a function of $1-r$. Equivalently, once $r$ is fixed, the magnitude of the KV perturbation is also bounded.

\begin{lemma}[Cosine similarity lower bound via distance]
\label{lem:cos_lb}
For any nonzero vectors $a$ and $b$, $1-\mathrm{CosSim}(a,b)\le \frac{\|a-b\|_2^2}{2\|a\|_2\|b\|_2}$.
\end{lemma}

\begin{proof}
Using the identity $\|a-b\|_2^2=\|a\|_2^2+\|b\|_2^2-2a^\top b$, we obtain $1-\mathrm{CosSim}(a,b)=\frac{\|a-b\|_2^2-(\|a\|_2-\|b\|_2)^2}{2\|a\|_2\|b\|_2}\le \frac{\|a-b\|_2^2}{2\|a\|_2\|b\|_2}$.
\end{proof}

Lemma~\ref{lem:cos_lb} allows us to convert a bound on the KV difference into a lower bound on the similarity, which is needed because $b(r)$ is defined through a similarity threshold.

\subsection{Theorem}\label{theorem}
\begin{theorem}[The shared prefix ratio determines a guaranteed depth]
\label{thm:r_to_b}
For each layer $\ell$ and step $t$, define $c_{\ell,t}\triangleq \frac{C_{\ell,t}^2}{2M_\ell^2}$ and $\mathrm{LB}_{\ell,t}(r)\triangleq 1-c_{\ell,t}g_{\ell,t}(1-r)^2$. Also define the worst-step lower bound $\underline{\mathrm{LB}}_\ell(r)\triangleq \min_t \mathrm{LB}_{\ell,t}(r)$. Then the following hold:
\begin{enumerate}[leftmargin=*]
    \item For every layer $\ell$ and step $t$, $\mathrm{sim}_{\ell,t}\ge \mathrm{LB}_{\ell,t}(r)$. Consequently, $\min_t \mathrm{sim}_{\ell,t}\ge \underline{\mathrm{LB}}_\ell(r)$.

    \item Define $b_{\mathrm{safe}}(r)\triangleq \max\Big\{ b\in\{0,1,\dots,L\}:\underline{\mathrm{LB}}_\ell(r)\ge\tau,\ \forall \ell\le b \Big\}$. Then $b_{\mathrm{safe}}(r)$ is determined by the shared prefix ratio $r$ through the lower bound $\underline{\mathrm{LB}}_\ell(r)$.

    \item The empirical shallow layer depth satisfies $b(r)\ge b_{\mathrm{safe}}(r)$.
\end{enumerate}
\end{theorem}

\begin{proof}
We prove the theorem in three steps.

\paragraph{Step 1: Lower-bounding the step-wise similarity.}
By Lemma~\ref{lem:kv_dev}, $\|u_{\ell,t}^{\mathrm{ref}}-u_{\ell}^{\$}\|_2 \le C_{\ell,t}g_{\ell,t}(1-r)$. We then convert this distance bound into a similarity bound. Applying Lemma~\ref{lem:cos_lb} with $a=u_{\ell}^{\$}$ and $b=u_{\ell,t}^{\mathrm{ref}}$, and using (A3), gives $1-\mathrm{sim}_{\ell,t} = 1-\mathrm{CosSim}\!\big(u_{\ell}^{\$},u_{\ell,t}^{\mathrm{ref}}\big) \le \frac{\|u_{\ell,t}^{\mathrm{ref}}-u_{\ell}^{\$}\|_2^2}{2M_\ell^2} \le \frac{C_{\ell,t}^2}{2M_\ell^2}\,g_{\ell,t}(1-r)^2 = c_{\ell,t}g_{\ell,t}(1-r)^2$. Rearranging yields $\mathrm{sim}_{\ell,t}\ge 1-c_{\ell,t}g_{\ell,t}(1-r)^2=\mathrm{LB}_{\ell,t}(r)$, which proves the first claim.

\paragraph{Step 2: Lower-bounding the worst-step similarity.}
Since the shallow layer depth is defined using the minimum similarity across steps, we take the minimum of the step-wise lower bound over all steps: $\min_t \mathrm{sim}_{\ell,t} \ge \min_t \mathrm{LB}_{\ell,t}(r)=\underline{\mathrm{LB}}_\ell(r)$. Thus, $\underline{\mathrm{LB}}_\ell(r)$ serves as a conservative lower bound on the similarity of layer $\ell$ across all steps.

\paragraph{Step 3: Determining a guaranteed depth from $r$.}
By definition, $b_{\mathrm{safe}}(r)$ is the deepest consecutive shallow layer depth such that every layer up to that depth satisfies the conservative lower bound $\underline{\mathrm{LB}}_\ell(r)\ge\tau$. Since $\underline{\mathrm{LB}}_\ell(r)$ is a function of $r$, the resulting depth $b_{\mathrm{safe}}(r)$ is also determined by $r$.

Finally, for every layer $\ell\le b_{\mathrm{safe}}(r)$, we have $\underline{\mathrm{LB}}_\ell(r)\ge\tau$. Combining this with the result from Step 2 gives $\min_t \mathrm{sim}_{\ell,t}\ge\tau$ for all $\ell\le b_{\mathrm{safe}}(r)$. Therefore, by the definition of $b(r)$, the empirical shallow layer depth satisfies $b(r)\ge b_{\mathrm{safe}}(r)$. This shows that the shared prefix ratio $r$ determines a guaranteed shallow layer depth through the lower bound $\underline{\mathrm{LB}}_\ell(r)$, which completes the proof.
\end{proof}

\section{Implementation Details}\label{app:3}
We implement \bicache on top of PyTorch. On LLaDA, a representative open-sourced DLM, we implement \bicache with $\sim$3K lines of code. To implement the KV store $H$ in layer-partitioned caching (\S\ref{sec:5.2}), we map each shared prefix to a deterministic 64-bit hash of its tokenized prefix sequence and use the resulting identifier as the shared prefix cache key. In our implementation, we use xxHash~\cite{xxhash}, a fast non-cryptographic hash function, to map each shared-prefix token ID sequence to a deterministic 64-bit identifier, which is then used as the cache key for lookup. At serving time, \bicache extracts the shared prefix, computes its hash, and retrieves the corresponding cached KVs on a hit; on a miss, it computes the shared prefix KVs once and inserts them into $H$.

Since GPU memory is finite, storing a large number of shared prefix KVs can cause out-of-memory issues. To manage the KV cache efficiently, we follow an eviction technique as in prior works \cite{vllm}. Specifically, we used a first-in, first-out policy, a standard approach that evicts the oldest cached KVs first. 

\begin{table*}[t]
\centering
\small
\setlength{\tabcolsep}{1pt}
\begin{tabular}{ccc cc cc cc}
\toprule
\multirow{2}{*}{\textbf{Method}}
& \multicolumn{2}{c}{\textbf{ARC-C}}
& \multicolumn{2}{c}{\textbf{GPQA}}
& \multicolumn{2}{c}{\textbf{MATH}}
& \multicolumn{2}{c}{\textbf{GSM8K}} \\
& \textbf{Thr.} & \textbf{Acc.}
& \textbf{Thr.} & \textbf{Acc.}
& \textbf{Thr.} & \textbf{Acc.}
& \textbf{Thr.} & \textbf{Acc.} \\
\midrule

No caching
& 17.28 & 87
& 15.87 & 27
& 15.97 & 20
& 10.8 & 60 \\

vLLM
& 33.87 {\scriptsize(+96.0\%)} & 1 {\scriptsize(-86.0\%)}
& 28.46 {\scriptsize(+79.3\%)} & 0 {\scriptsize(-27.0\%)}
& 32.85 {\scriptsize(+105.7\%)} & 0 {\scriptsize(-20.0\%)}
& 30.55 {\scriptsize(+182.9\%)} & 42 {\scriptsize(-18.0\%)} \\

\textbf{\bicache}
& \textbf{32.19} {\scriptsize\textbf{(+86.3\%)}} & \textbf{85} {\scriptsize\textbf{(-2.0\%)}}
& \textbf{27.43} {\scriptsize\textbf{(+72.8\%)}} & \textbf{27} {\scriptsize\textbf{(+0.0\%)}}
& \textbf{31.43} {\scriptsize\textbf{(+96.8\%)}} & \textbf{24} {\scriptsize\textbf{(+4.0\%)}}
& \textbf{28.25} {\scriptsize\textbf{(+161.6\%)}} & \textbf{56} {\scriptsize\textbf{(-4.0\%)}} \\
\midrule

Fast-dLLM
& 22.25 & 85
& 21.82 & 26
& 16.97 & 23
& 11.98 & 72 \\

\textbf{\shortstack[c]{\bicache\\+Fast-dLLM}}
& \raisebox{0.3\normalbaselineskip}[0pt][0pt]{\makecell[c]{\textbf{32.33}\\[-1pt]{\scriptsize\textbf{(+45.3\%)}}}}
& \raisebox{0.3\normalbaselineskip}[0pt][0pt]{\makecell[c]{\textbf{87}\\[-1pt]{\scriptsize\textbf{(+2.0\%)}}}}
& \raisebox{0.3\normalbaselineskip}[0pt][0pt]{\makecell[c]{\textbf{32.01}\\[-1pt]{\scriptsize\textbf{(+46.7\%)}}}}
& \raisebox{0.3\normalbaselineskip}[0pt][0pt]{\makecell[c]{\textbf{25}\\[-1pt]{\scriptsize\textbf{(-1.0\%)}}}}
& \raisebox{0.3\normalbaselineskip}[0pt][0pt]{\makecell[c]{\textbf{32.57}\\[-1pt]{\scriptsize\textbf{(+91.9\%)}}}}
& \raisebox{0.3\normalbaselineskip}[0pt][0pt]{\makecell[c]{\textbf{17}\\[-1pt]{\scriptsize\textbf{(-6.0\%)}}}}
& \raisebox{0.3\normalbaselineskip}[0pt][0pt]{\makecell[c]{\textbf{29.12}\\[-1pt]{\scriptsize\textbf{(+143.1\%)}}}}
& \raisebox{0.3\normalbaselineskip}[0pt][0pt]{\makecell[c]{\textbf{74}\\[-1pt]{\scriptsize\textbf{(+2.0\%)}}}} \\
\bottomrule
\end{tabular}

\caption{Throughput (tokens/s) and accuracy (\%) with LLaDA-1.5. We use the same benchmarks and comparison settings as in \S\ref{sec:6}. The upper part reports main improvements; the lower part shows seamless integration results. Percentages of \bicache are relative to no caching (upper part) and Fast-dLLM (lower part). Bold: \bicache results.}
\label{tab:5}
\end{table*}

\begin{table*}[t]
\centering
\small
\setlength{\tabcolsep}{1pt}
\begin{tabular}{c c c c c c}
\toprule
Metric
& No caching
& vLLM
& \textbf{\bicache}
& Fast-dLLM
& \textbf{\bicache + Fast-dLLM} \\
\midrule

Throughput
& 55.7
& 93.6 {\scriptsize(+68.0\%)}
& \textbf{91.4 {\scriptsize(+64.1\%)}}
& 65.3
& \textbf{98.4 {\scriptsize(+50.7\%)}} \\

Accuracy
& 85.9
& 0.0 {\scriptsize(-85.9\%)}
& \textbf{95.0 {\scriptsize(+9.1\%)}}
& 84.5
& \textbf{82.0 {\scriptsize(-2.5\%)}} \\
\bottomrule
\end{tabular}

\caption{Throughput (tokens/s) and accuracy (\%) with a low shared prefix ratio. We use ARC-Challenge-Chat and same comparison settings as in \S\ref{sec:6}. Bold: \bicache results.}
\label{tab:6}
\end{table*}
\begin{table*}[t]
\centering

\begin{minipage}[t]{0.545\linewidth}
\centering
\small
\setlength{\tabcolsep}{1pt}
\resizebox{\linewidth}{!}{%
\begin{tabular}{c cc cc cc cc cc}
\toprule
\multirow{2}{*}{\textbf{Method}}
& \multicolumn{2}{c}{\textbf{0.91}}
& \multicolumn{2}{c}{\textbf{0.93}}
& \multicolumn{2}{c}{\textbf{0.95}}
& \multicolumn{2}{c}{\textbf{0.97}}
& \multicolumn{2}{c}{\textbf{0.99}} \\
& \textbf{Thr.} & \textbf{Acc.}
& \textbf{Thr.} & \textbf{Acc.}
& \textbf{Thr.} & \textbf{Acc.}
& \textbf{Thr.} & \textbf{Acc.}
& \textbf{Thr.} & \textbf{Acc.} \\
\midrule

\textbf{\bicache}
& 92.4 & 39.1
& 92.1 & 48.4
& 92.0 & 45.9
& \textbf{91.4} & \textbf{55.5}
& 91.2 & 53.8 \\
\midrule

\textbf{\shortstack[c]{\bicache\\+Fast-dLLM}}
& \raisebox{0.35\normalbaselineskip}{103.2}
& \raisebox{0.35\normalbaselineskip}{46.3}
& \raisebox{0.35\normalbaselineskip}{103.2}
& \raisebox{0.35\normalbaselineskip}{57.5}
& \raisebox{0.35\normalbaselineskip}{104.3}
& \raisebox{0.35\normalbaselineskip}{55.2}
& \raisebox{0.35\normalbaselineskip}{\textbf{102.5}}
& \raisebox{0.35\normalbaselineskip}{\textbf{63.9}}
& \raisebox{0.35\normalbaselineskip}{103.2}
& \raisebox{0.35\normalbaselineskip}{65.0} \\

\bottomrule
\end{tabular}%
}
\subcaption{Impact of threshold $\tau$.}
\label{tab:7a}
\end{minipage}
\hfill
\begin{minipage}[t]{0.445\linewidth}
\centering
\small
\setlength{\tabcolsep}{1pt}
\resizebox{\linewidth}{!}{%
\begin{tabular}{c cc cc cc cc}
\toprule
\multirow{2}{*}{\textbf{Method}}
& \multicolumn{2}{c}{\textbf{8}}
& \multicolumn{2}{c}{\textbf{16}}
& \multicolumn{2}{c}{\textbf{32}}
& \multicolumn{2}{c}{\textbf{64}} \\
& \textbf{Thr.} & \textbf{Acc.}
& \textbf{Thr.} & \textbf{Acc.}
& \textbf{Thr.} & \textbf{Acc.}
& \textbf{Thr.} & \textbf{Acc.} \\
\midrule

\textbf{\bicache}
& 89.1 & 55.5
& \textbf{91.4} & \textbf{55.5}
& 92.7 & 54.9
& 93.4 & 53.5 \\
\midrule

\textbf{\shortstack[c]{\bicache\\+Fast-dLLM}}
& \raisebox{0.35\normalbaselineskip}{99.6}
& \raisebox{0.35\normalbaselineskip}{64.8}
& \raisebox{0.35\normalbaselineskip}{\textbf{102.5}}
& \raisebox{0.35\normalbaselineskip}{\textbf{63.9}}
& \raisebox{0.35\normalbaselineskip}{104.1}
& \raisebox{0.35\normalbaselineskip}{61.9}
& \raisebox{0.35\normalbaselineskip}{104.7}
& \raisebox{0.35\normalbaselineskip}{61.9} \\

\bottomrule
\end{tabular}%
}
\subcaption{Impact of refresh interval $\Delta$.}
\label{tab:7b}
\end{minipage}

\caption{Additional experiments evaluating the impact of the threshold ($\tau$) and refresh interval ($\Delta$) on throughput (tokens/s) and accuracy (\%) on GSM8K. Bold text denotes the selected configuration.}
\label{tab:7}
\end{table*}

\begin{table}[t]
\centering
\small
\begin{tabular}{c|c}
\toprule

Benchmark & Shared prefix ratio \\

\midrule

ARC-C & 88.2\% \\
GPQA & 75.9\% \\
MATH & 90.1\% \\
GSM8K & 94.5\% \\

\bottomrule
\end{tabular}

\caption{Average shared prefix ratio in each benchmark.}
\label{tab:8}
\end{table}

\section{Additional Analysis and Experiments}\label{app:D}

Here, we present additional experiment results beyond those reported in \S\ref{sec:6}:
\squishlist
    \item Appendix~\ref{app:sys}: Robustness of \bicache to system prompt changes
    \item Appendix~\ref{app:D2}: Profiling overhead for different values of $M$
    \item Appendix~\ref{app:D3}: Generalizability of \bicache to a different DLM architecture, LLaDA-1.5 \cite{llada-1.5}
    \item Appendix~\ref{app:D4}: Performance of \bicache under a low shared prefix ratio
    \item Appendix~\ref{app:D5}: Parameter analysis on GSM8K
    \item Appendix~\ref{app:ratio}: Shared prefix ratios of the evaluation benchmarks
    \item Appendix~\ref{app:math}: Analysis of the relatively low absolute accuracy on MATH
    \item Appendix~\ref{app:rob}: Robustness of \bicache to shallow-layer depth estimation errors
\squishend

\begin{figure}[t]
\centering
\begin{minipage}[t]{0.4\linewidth}
    \centering
    \includegraphics[width=\linewidth]{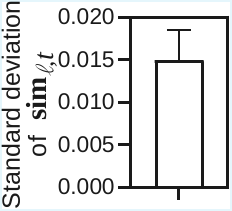}
    \caption{Standard deviation under system prompt changes.}
    \label{fig:11}
\end{minipage}
\hfill
\begin{minipage}[t]{0.58\linewidth}
    \centering
    \includegraphics[width=\linewidth]{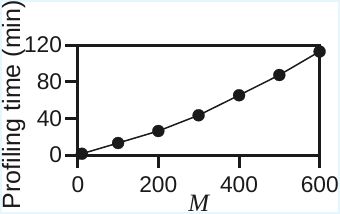}
    \caption{Profiling time changes for $M$.}
    \label{fig:12}
\end{minipage}
\end{figure}

\subsection{Analysis of $b$ for System Prompt Changes}\label{app:sys}
In \S\ref{sec:4.2}, we show that the shared prefix ratio $r$ is strongly related to the shallow layer depth $b$ across different tasks. One may wonder whether this relationship still holds when the system prompt itself changes.
In practice, system prompts are typically identical across requests. However, in some cases, a small subset of tokens may vary (e.g., timestamps, metadata fields, or minor changes in instruction wording). In such cases, only a few tokens within the shared prefix differ while most of the prefix remains identical across requests \cite{preble}.

We investigate this effect as follows. For each $r \in \{0.50, 0.51, {\ldots}, 1.00\}$, we generate 500 samples while keeping $r$ fixed. Using the system prompt in Figure~\ref{fig:10} as a source template, we instruct an LLM to generate diverse variants by varying role descriptions, safety instructions, tone specifications, and output-format constraints while avoiding near-duplicate rewrites. We further instruct that at least half of the tokens in each generated prompt differ from those in the source template. We then measure the standard deviation of $\mathrm{sim}_{\ell,t}$ across different $r$ values.
A lower standard deviation indicates that $\mathrm{sim}_{\ell,t}$ remains stable even when the system prompt changes, which suggests that $b$ is not sensitive to such changes.

Fig. \ref{fig:11} shows the average standard deviation of $\mathrm{sim}_{\ell,t}$ across ratios, with error bars. The deviation remains low (0.014 on average), indicating that system prompt changes have little effect on $b$. Combined with the observation in \S\ref{sec:4.2}, this supports our design choice of profiling $b$ only once offline for each model.

\subsection{Profiling Time}\label{app:D2}
In \S\ref{sec:6.6}, we evaluate which value of $M$ is appropriate. Here, we present the profiling overhead for different values of $M$. Fig.~\ref{fig:12} shows the time required for profiling as the number of profiling requests $M$ increases. For each value of $M$, \bicache profiles $M$ requests for each shared prefix ratio $r \in \{0.01,{\ldots},1.00\}$. Thus, for example, $M{=}600$ means profiling a total of $600 \times 100$ requests across all shared prefix ratios.

As shown in the figure, the profiling time linearly increases as $M$ increases. In particular, when $M{=}500$ that we use, profiling takes 87.5 minutes. Although this cost is nontrivial, profiling is a one-time offline procedure performed at model deployment rather than during online serving, and it does not need to be repeated when the task changes (\S\ref{sec:4.2}) or when the shared prefix changes, such as through a system prompt update (\S\ref{app:sys}). Therefore, it does not affect the runtime latency of user requests. Moreover, the profiling cost can be further reduced by profiling only a subset of representative $r$ values or by using a more efficient search algorithm instead of exhaustively profiling all ratios.

\subsection{Generalizability on Other DLM Architecture}\label{app:D3}
To evaluate the effectiveness of \bicache in other DLMs, we conduct a new experiment on LLaDA-1.5 \cite{llada-1.5}, which has a distinct DLM architecture rather than being a minor variant of the original LLaDA model, since LLaDA-1.5 differs from LLaDA in its training recipe, such as applying variance-reduced preference optimization for preference alignment for ELBO-based preference optimization. This allows us to examine whether \bicache's shared prefix KV reuse principle holds beyond the original LLaDA model.

Table~\ref{tab:5} shows the experiment results. First, across the four benchmarks, \bicache improves throughput over no caching by 72.8\% on GPQA to 161.6\% on GSM8K, while closely matching the throughput of vLLM. Second, \bicache maintains accuracy within 4\% of the no caching, which avoids the severe accuracy collapse observed with vLLM. The results are consistent with our results in \S\ref{sec:6} and demonstrate that \bicache generalizes effectively to a different DLM architecture.

\subsection{Results on Low Shared Prefix Ratio}\label{app:D4}
We prepare low shared prefix ratio setting for ARC-Challenge-Chat, a representative task. Specifically, we reduce the shared prefix ratio from 88\% in \S\ref{sec:6} to 22\%. Table \ref{tab:6} shows the throughput and accuracy under the low shared prefix ratio. \bicache improves throughput over no caching by 64.1\% while avoiding the accuracy collapse observed with vLLM. In addition, combining \bicache with Fast-dLLM achieves the highest throughput, demonstrating that \bicache remains effective under limited prefix sharing and is complementary to intra-request acceleration.

\subsection{Parameter Analysis on GSM8K}\label{app:D5}
In \S\ref{sec:6.6}, we conduct a parameter analysis on ARC-Challenge-Chat. Here, we extend the parameter analysis to another benchmark.
Table~\ref{tab:7} reports GSM8K results for different threshold $\tau$ and refresh interval $\Delta$ settings. All other experiment settings follow \S\ref{sec:6}.

\noindent\textbf{Threshold.}
Table~\ref{tab:7a} varies the threshold $\tau$ while fixing $\Delta = 16$. Increasing $\tau$ improves accuracy while throughput remains largely stable. Specifically, $\tau = 0.97$ achieves high accuracy for both \bicache and \bicache + Fast-dLLM while maintaining comparable throughput.

\noindent\textbf{Refresh Interval.}
Table~\ref{tab:7b} varies the refresh interval $\Delta$ while fixing $\tau = 0.97$. Decreasing $\Delta$ improves accuracy but reduces throughput. Specifically, $\Delta = 16$ achieves high accuracy for both \bicache and \bicache + Fast-dLLM while maintaining comparable throughput.

\subsection{Shared Prefix Ratio in Benchmarks}\label{app:ratio}
Here, we report the shared prefix ratios of the benchmarks used in \S\ref{sec:3.2} and \S\ref{sec:6}. Table~\ref{tab:8} shows the average shared prefix ratio for each benchmark. The ratio varies across benchmarks because the benchmark-specific user prompt lengths differ, while we use the same fixed system prompt. GSM8K has the highest shared prefix ratio at 94.5\%, while GPQA has the lowest at 75.9\%. Overall, these values are broadly comparable to the real-world range of shared prefix ratios discussed in \S\ref{sec:3.1}, although some benchmarks fall slightly below that range. In light of the results in \S\ref{sec:3.2}, these results suggest that a higher shared prefix ratio tends to yield larger throughput improvements.

\subsection{Discussion of MATH Accuracy}\label{app:math}
As shown in Tables~\ref{tab:1}, \ref{tab:2}, and \ref{tab:5}, all evaluated methods achieve relatively low absolute accuracy on MATH. This trend is not specific to \bicache; recent DLM acceleration studies, including Fast-dLLM \cite{fastdllm} and ES-dLLM \cite{es-dllm}, have reported similarly low accuracy. This is because MATH is a challenging competition-level benchmark that requires multi-step symbolic reasoning, making it substantially more difficult than arithmetic benchmarks such as GSM8K. \bicache maintains accuracy comparable to no caching across all three tables, indicating that the low absolute accuracy reflects the difficulty of MATH for current DLMs rather than degradation caused by \bicache.

\subsection{Robustness to Shallow Layer Depth Estimation Errors}\label{app:rob}
We additionally measure the error introduced by estimating the shallow layer depth $b$ using the profiling mechanism of \bicache that happens only once per model (\S\ref{sec:5.1}), rather than measuring it separately for each request. On the 32-layer LLaDA model, the profiling-based estimate differs from the per-request measurement by only 0.9 layers on average, corresponding to 2.8\% of the total depth. Thus, the estimated depth is typically within one layer of the per-request measurement. Note that \bicache delivers substantial throughput and latency gains (as reported in \S\ref{tab:3}), which indicates that this small estimation error has little impact on serving performance.

\section{AI Assistants in Research or Writing}
We use AI assistants, i.e., Gemini and ChatGPT, for proofreading support, including grammar, typo, and wording checks throughout the manuscript. The assistants are used solely to identify errors in our original writing and are not used to develop the core research ideas, methods, experiments, results, or conclusions; all such content is created and verified by the authors.

\end{document}